\documentclass{article}

\PassOptionsToPackage{numbers, compress}{natbib}

\usepackage[preprint]{neurips_2026}

\usepackage[utf8]{inputenc} 
\usepackage[T1]{fontenc}    
\usepackage{hyperref}       
\usepackage{url}            
\usepackage{booktabs}       
\usepackage{amsfonts}       
\usepackage{nicefrac}       
\usepackage{microtype}      
\usepackage{xcolor}         
\usepackage{graphicx}
\usepackage{float}
\usepackage{mathtools}
\usepackage{todonotes}
\usepackage{enumitem}
\usepackage{subcaption}
\usepackage{amsthm}
\usepackage{algorithm}
\usepackage{algorithmic}
\usepackage{multirow}
\usepackage{wrapfig}

\theoremstyle{plain}
\newtheorem{proposition}{Proposition}

\theoremstyle{definition}

\theoremstyle{plain}

\newcommand{\vx}{\boldsymbol{x}}
\newcommand{\vc}{\boldsymbol{c}}
\newcommand{\vz}{\boldsymbol{z}}

\newcommand{\vk}{\boldsymbol{k}}
\newcommand{\vv}{\boldsymbol{v}}
\newcommand{\dmodel}{d_{\text{model}}}

\newcommand{\KL}{D_{\mathrm{KL}}}
\newcommand{\pos}{\mathtt{pos}}

\title{NEST: Nested Event Stream Transformer for Sequences of Multisets}

\author{
  Minghui Sun\thanks{equally contributed}$~~^{1}$ \quad Haoyu Gong\footnotemark[1]$~~^{2}$ \quad Xingyu You$^{1}$ 
  \And
  Jillian Hurst$^{3}$ \quad Benjamin Goldstein$^{1}$ \quad Matthew Engelhard$^{1}$ \\
  \\
  $^1$Department of Biostatistics \& Bioinformatics, Duke University \\
  $^2$Department of Biomedical Engineering, Duke University \\
  $^3$Department of Pediatrics, Duke University \\
  \texttt{\{minghui.sun, m.engelhard\}@duke.edu} \\
}

\begin{document}

\maketitle

\begin{abstract}
    Event stream data often exhibit hierarchical structure in which multiple events co-occur, resulting in a sequence of multisets (i.e., bags of events). In electronic health records (EHRs), for example, medical events are grouped into a sequence of clinical encounters with well-defined temporal structure, but the order and timing of events within each encounter may be unknown or unreliable. Most existing foundation models (FMs) for event stream data flatten this hierarchy into a one-dimensional sequence, leading to (i) computational inefficiency associated with dense attention and learning spurious within-set relationships, and (ii) lower-quality set-level representations from heuristic post-training pooling for downstream tasks. Here, we show that preserving the original hierarchy through interleaved within-set and cross-set encoders provides a useful inductive bias that improves both computational efficiency and representation quality. We then introduce \textbf{N}ested \textbf{E}vent \textbf{S}tream \textbf{T}ransformer (\textbf{NEST}), a FM for event streams comprised of sequences of multisets. Building on this architecture, we formulate Masked Set Modeling (MSM), a paradigm that promotes improved set-level representation learning. Experiments on real-world datasets show that NEST captures real-world dynamics while improving both pretraining efficiency and downstream performance. Our code will be publicized upon acceptance.
\end{abstract}

\section{Introduction}
While originally designed for natural language sequence data~\citep{vaswani2017attention, radford2018improving, devlin2019bert}, the Transformer's ability to model complex categorical data dependencies and dynamics has inspired its widespread application to general event stream data, which are continuous-time sequences of complex, multi-modal events \citep{mcdermott2023event}. Transformer-based foundation models (FMs) have emerged as a transformative paradigm across diverse domains, enabling rapid adaptation to specialized downstream tasks through pretraining on large-scale data. In healthcare, electronic health records (EHR) data are represented as sequences of clinical events, from which FMs extract deep representations for clinical decision support~\citep{li2020behrt, wornow2023ehrshot}. The same framework extends to commercial applications, where user purchase histories are encoded for next-item and next-basket recommendation~\citep{kang2018self, sun2019bert4rec}, and to financial domains, where real-time transaction streams are converted to deep features for fraud prevention~\citep{hu2023bert4eth}.

We first define the term \emph{multiset}, which is a \emph{set with multiplicity}: a collection that, unlike a standard set, permits duplicate elements. Event stream data frequently possesses a hierarchical structure in which multiple events co-occur, forming a sequence of multisets. In the EHR domain, for example, medical events are grouped into encounters with clear temporal ordering, but the order of events \textit{within} each encounter can be noisy, unreliable, or otherwise unimportant when modeling extended patient timelines. Within each encounter, events can be time-ordered when fine-grained temporal resolution is available (e.g., ICU settings) or treated as unordered otherwise (e.g., long-term routine clinical visits). The same structure recurs beyond clinical data: a customer's purchase history is a sequence of orders, each an unordered basket of products. We term this data structure a \textbf{seq}uence of event multi\textbf{set}s (\texttt{SeqSet}). Modeling \texttt{SeqSet} structure explicitly enables both efficient computation through structured sparsity and direct learning of set-level representations without post-hoc pooling.

\begin{figure}[t]
    \centering
    \includegraphics[width=0.65\linewidth]{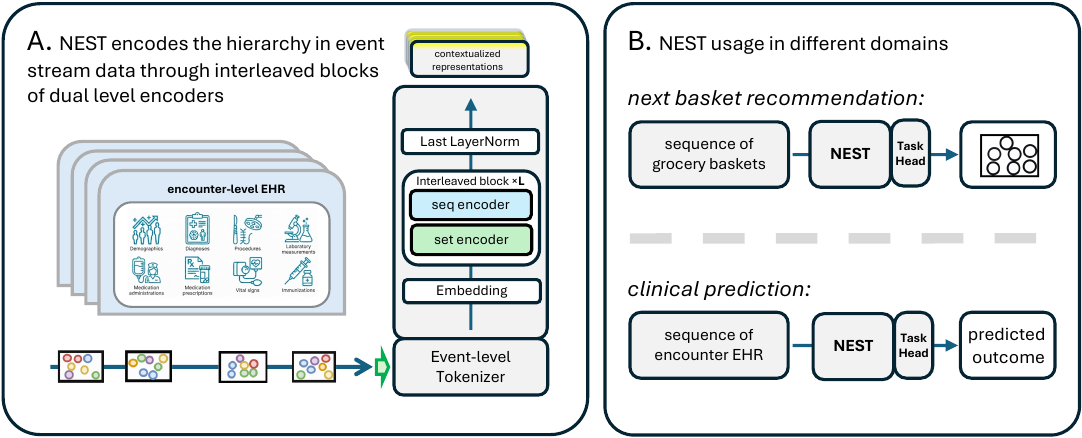}
    \caption{NEST is a foundation model for event streams structured as sequences of multisets (\texttt{SeqSet}). \textbf{(A)} Each NEST layer is an interleaved block consisting of a Set-Wise Encoder (SWE) followed by a Cross-Set Encoder (CSE), repeated $L$ times. \textbf{(B)} The same architecture instantiates across domains.
}
    \label{fig:nest-fig1}
\end{figure}

Most FMs for event-stream data adopt flat Transformer architectures~\citep{li2020behrt, rasmy2021med, wornow2023ehrshot, odgaard2024core, renc2024ethos}, which do not explicitly model the hierarchical structure of event streams. GPT-style causal modeling~\citep{mcdermott2023event, renc2024ethos} imposes a fixed token order on locally permutation-invariant groups of events, while BERT-style full attention~\citep{rasmy2021med, odgaard2024core} preserves all pairwise interactions but at computing cost that grows quadratically with context size. More fundamentally, both paradigms learn representations only at the token level through MLM or NTP; encounter-level representations, the clinically meaningful unit of analysis, are only recovered through post-hoc pooling, which prior work has shown to be suboptimal~\citep{reimers2019sentence, gao2021simcse}. The principle of learning representations at semantically coherent units rather than individual tokens has been advocated in language modeling~\citep{barrault2024lcm} and object-centric visual reasoning~\citep{locatello2020object}. We instantiate this idea for clinical event streams through joint architectural and pretraining changes, treating encounters as the natural unit of representation.

Building on hierarchical models for text~\citep{yang2016hierarchical, chalkidis2022hat} and event streams~\citep{hu2019sets2sets, li2022hibehrt}, we introduce \textbf{N}ested \textbf{E}vent \textbf{S}tream \textbf{T}ransformer (\textbf{NEST}), a foundation model that structurally aligns with \texttt{SeqSet} data. Each NEST layer pairs a set-wise encoder (SWE) for within-set token interactions with a cross-set encoder (CSE) for sequence-level dynamics, with the two encoders interleaved layer-by-layer rather than stacked in two phases. This interleaving yields \textit{dual-level contextualization}, providing (i) structurally induced sparse attention and (ii) direct set-level representations without post-hoc pooling.

Apart from the architectural modification, we also propose a separate pretraining objective, Masked Set Modeling (MSM), that is tailored for sequences of multisets and complements the standard MLM objective during NEST training. The MSM objective enables NEST to learn ready-to-use set representations, eliminating heuristic post-training pooling.

We conduct extensive experiments across diverse event stream datasets from multiple real-world domains to evaluate NEST's effectiveness. NEST achieves state-of-the-art performance and improved computational efficiency compared to Transformer models trained on flattened event streams across multiple downstream tasks from different domains. In addition, ablation studies show that the proposed MSM objective enhances set-level representation learning. We argue that the right inductive bias for event-stream foundation models is to make the architecture and pretraining objective both reflect the data's nested set structure, rather than flattening it away. Concretely, our contributions are:

\begin{itemize}[topsep=2pt, itemsep=2pt, parsep=0pt, partopsep=0pt]
    \item \textbf{Architectural primitive for SeqSet.} We introduce \textbf{NEST}, a hierarchical Transformer foundation model whose interleaved within-set and cross-set encoders respect SeqSet's opposed symmetries: intra-set exchangeability and inter-set temporal order.
    
    \item \textbf{Architectural rationale via bypass paths.} We characterize (Proposition~\ref{prop:bypass}) how interleaved composition preserves token-level information that sequential designs irrecoverably discard at the SWE-to-CSE phase boundary, with a topology-isolating simulation confirming the predicted advantage.
    
    \item \textbf{Set-level pretraining objective.} We formulate \textbf{Masked Set Modeling (MSM)}, a permutation-invariant objective that complements MLM by training \texttt{[CLS]} to predict its multiset's token distribution, which no prior EHR foundation model has trained for explicitly.
\end{itemize}

\section{Related Works}
\textbf{Hierarchical Transformer and Structured Attention.} \quad Hierarchical architectures have proven effective for modeling structured text~\citep{chalkidis2019neural, chalkidis2022hat, he2024hdt}, achieving efficiency through structurally induced sparsity. Alternative approaches to efficient attention include sliding-window patterns~\citep{beltagy2020longformer, zaheer2020big}, kernel-based or landmark approximations~\citep{choromanski2020rethinking, xiong2021nystromformer}, and learned latent bottlenecks~\citep{jaegle2021perceiver}. However, these methods trade off attention fidelity for efficiency without exploiting data-inherent grouping structure.

\textbf{EHR Foundation Models.}\quad Early EHR foundation models adapted BERT-style MLM~\citep{li2020behrt, rasmy2021med}, with subsequent work exploring autoregressive~\citep{wornow2023ehrshot, renc2024ethos}, time-to-event~\citep{steinberg2023motor}, and contrastive~\citep{jeong2023ebcl} objectives. Architecturally, CORE-BEHRT~\citep{odgaard2024core} optimizes the flat Transformer but collapses all multisets into a single sequence. Hierarchical EHR models adopt sequential two-stage designs: Hi-BEHRT~\citep{li2022hibehrt} uses fixed sliding windows that can fragment multisets; GT-BEHRT~\citep{poulain2024graph} decouples a per-visit graph encoder from a sequence Transformer during pretraining; and HEART~\citep{huang2024heart} enriches within-encounter representations with additional pairwise relational embeddings. NEST instead interleaves within-set and cross-set processing, respects multiset boundaries, and introduces MSM for explicit set-level representation learning.

\textbf{Set-level Representation.}\quad Naive pooling strategies such as mean-pooling or directly \texttt{[CLS]} embeddings yield suboptimal set-level representations~\citep{reimers2019sentence}. Principled alternatives include Pooling by Multihead Attention~\citep{lee2019set} for permutation-invariant set summarization and contrastive methods for refining \texttt{[CLS]} embeddings~\citep{gao2021simcse, chuang2022diffcse}, though contrastive objectives provide only implicit pressure for the token to summarize set contents. NEST instead explicitly supervises set-level representations through a set prediction objective (MSM, Section~\ref{sec:msm}), using the \texttt{[CLS]} token as the set-level readout.

\section{Method}
\subsection{Event Stream as a Sequence of Sets}

We model an event stream as a \textbf{time-ordered sequence of multisets}, where each multiset represents a collection of events that may be time-ordered (e.g., short ICU settings) or treated as unordered (e.g., long-term routine clinical visits). Formally:
\begin{equation}
    \big((\mathcal{X}_i, T_i)\big)_{i=1}^{m},~~
    \mathcal{X}_i= \big\{(\vc_i, T_i)\big\} \cup \big\{(\mathtt{x}_{ij},t_{ij})\big\}_{j=1}^{n}
\end{equation}
where $\mathcal{X}_i$ is the $i$-th event multiset (e.g., clinical visit), padded and truncated to fixed size $n$, and $m$ is the number of multisets. Each set includes a learnable \texttt{[CLS]} token $\vc_i$ that serves as the set-level readout. We assign $\vc_i$ the set timestamp $T_i$ so that the CSE can encode inter-set time ordering via RoPE (Section~\ref{sec:dual-contextualize}). The optional fine-grained timestamps $t_{ij}$ capture within-set temporal structure when available.

We define embedding notation used throughout. Let $[\cdot]$ denote column-wise concatenation. $X_i = [\{\vc_i\} \cup \{\vx_{ij}\}_{j=1}^{n}] \in \mathbb{R}^{\dmodel \times (n+1)}$ represents the token embeddings of $\mathcal{X}_i$. At the $l$-th NEST layer, $\vc_i^{(l)}$ and $Z_i^{(l)} = [\{\vz_{ij}^{(l)}\}]$ denote the contextualized \texttt{[CLS]} and non-\texttt{[CLS]} hidden states, respectively. We drop the superscript for final-layer outputs: $\vc_i \coloneqq \vc_i^{(L)}$, $Z_i \coloneqq Z_i^{(L)}$.

\subsection{Overview of NEST}

Each NEST layer consists of a Set-Wise Encoder (SWE) followed by a Cross-Set Encoder (CSE), interleaved layer-by-layer. Although prior work has empirically explored such interleaving for long documents~\citep{chalkidis2022hat}, the architectural consequences for sequence-of-multisets data and set-level pretraining have not been characterized. We provide an information-theoretic analysis in Section~\ref{sec:bypass} and simulation study in Section~\ref{sec:simulation}. Within each layer, the two encoders perform complementary roles: tokens attend only within their own multiset via SWE; cross-set information flows exclusively through the \texttt{[CLS]} tokens via CSE. Although SWE and CSE each apply dense attention internally, their composition over the full token set induces a sparse attention pattern (Figure~\ref{fig:nest}, right). \emph{Across the full stack, global context emerges through the \texttt{[CLS]}-mediated bridge: CSE injects cross-set context into \texttt{[CLS]}, and the subsequent SWE distributes it back to set members}. This design (i) improves computing efficiency through structural sparsity, and (ii) produces set-level representations via dual-level contextualization without post-hoc pooling.

\begin{figure}[t]
    \centering
    \includegraphics[width=0.7\linewidth]{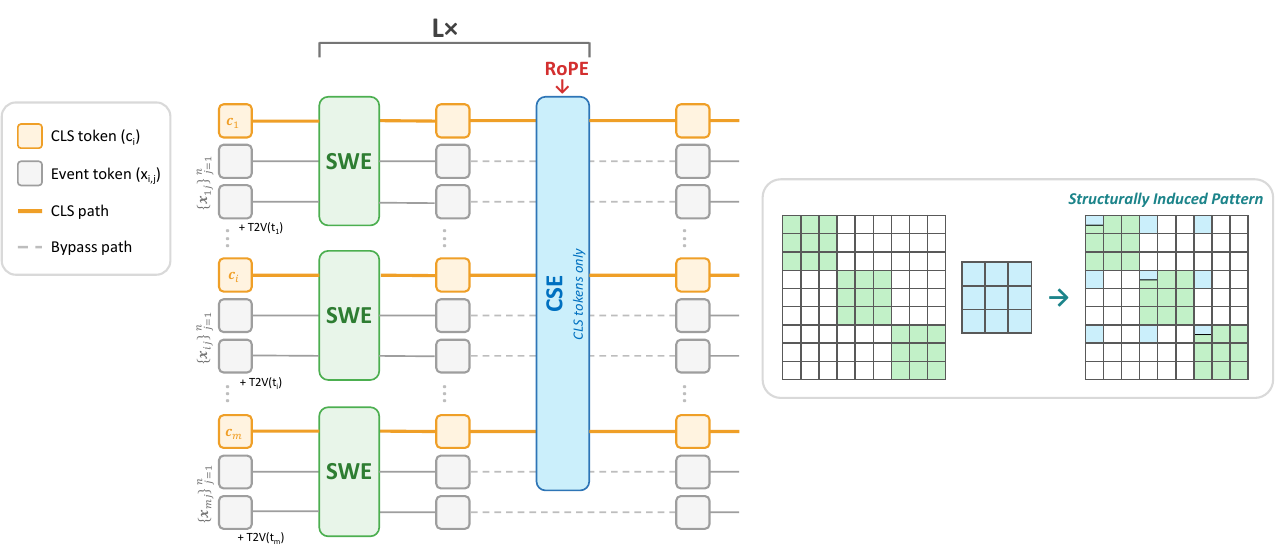}
    \caption{NEST models event stream data as a sequence of event multisets (\texttt{SeqSet}). At each of the $L$ layers, a SWE contextualizes tokens within each set, followed by a CSE that contextualizes \texttt{[CLS]} tokens across sets with RoPE-encoded inter-set time. The two encoders are interleaved. Non-\texttt{[CLS]} token hidden states are propagated through bypass paths so that within-set information remains accessible at every layer. Their composition yields a structurally induced sparse attention pattern.}
    \label{fig:nest}
\end{figure}

\subsection{Dual-level Contextualization}
\label{sec:dual-contextualize}
We use modern Transformer components~\citep{warner2025modernbert} in all SWE/CSE blocks. Within each multiset, interactions are restricted to in-set tokens, with the \texttt{[CLS]} aggregating set-level features via set-wise contextualization. Across the sequence, \texttt{[CLS]} communicate through cross-set contextualization.

\textbf{Set-wise}. The set-wise contextualization occurs within each set independently. Given the multiset $\mathcal{X}_i$,
\begin{align}
    \vz^{(l+1)}_{ij} &\leftarrow \vz^{(l)}_{ij} + \sum_a W_O^{(l,a)}V_i^{(l,a)}\text{Softmax}\left(\frac{1}{\sqrt{d_k}}K_{i}^{(l,a)\top} q_{ij}^{(l,a)}\right) \\
    \notag V_{i}^{(l,a)} &= W_V^{(l,a)}X_i^{(l)} =[\{\vv_{ij}^{(l,a)}\}_{j=1}^n] \\
    \notag K_{i}^{(l,a)} &= W_K^{(l,a)}X_i^{(l)} =[\{\vk_{ij}^{(l,a)}\}_{j=1}^n]
\end{align}
$V_i^{(l,a)}$ and $K_i^{(l,a)}$ are the layer- and head-specific value and key matrices for the tokens in $\mathcal{X}_i$, respectively, and $q_{ij}^{(l,a)} = W_Q^{(l,a)} \vz_{ij}^{(l)}$ is the layer- and head-specific query vector of token $j$. $W_O^{(l,a)} \in \mathbb{R}^{\dmodel \times d_k}$ and $W_Q^{(l,a)}, W_K^{(l,a)}, W_V^{(l,a)} \in \mathbb{R}^{d_k \times\dmodel}$ are the head-specific projection matrices in multi-head attention of SWE. SWE omits positional encoding by default, enforcing permutation invariance over intra-set tokens. When fine-grained timestamps are available (e.g., ICU settings), an optional intra-set RoPE can be applied to keys analogously to CSE, encoding within-set temporal order via $t_{ij}$. Inside each NEST block, SWE is shared across all multisets in the sequence.

\textbf{Cross-set}. The cross-set contextualization operates solely on \texttt{[CLS]} embeddings $\{\vc_i^{(l)}\}_{i=1}^{m}$. The hidden states of the other tokens bypass CSE.
\begin{align}
    \vc^{(l+1)}_{i} &\leftarrow \vc^{(l)}_{i} + \sum_a \tilde{W}_O^{(l,a)} \tilde{V}^{(l,a)} \, \text{Softmax}\left(\frac{1}{\sqrt{d_k}} \tilde{K}^{(l,a)\top} R(\pos_i) \tilde{q}_{i}^{(l,a)}\right) \\
    \notag \tilde{V}^{(l,a)} &= \tilde{W}_V^{(l,a)}[\{\vc_i^{(l)}\}_{i=1}^m] = \left[\{\tilde{\vv}_{i}^{(l,a)}\}_{i=1}^m\right] \\
    \notag \tilde{K}^{(l,a)} &= \text{RoPE}\!\left(\tilde{W}_K^{(l,a)}[\{\vc^{(l)}_{i}\}_{i=1}^m]\right) = \left[\{R(\pos_i) \tilde{\vk}_{i}^{(l,a)}\}_{i=1}^m\right]
\end{align}
where $\tilde{q}_{i}^{(l,a)} = \tilde{W}_Q^{(l,a)} \vc_i^{(l)}$ is the query vector for the \texttt{[CLS]} token in $\mathcal{X}_i$, $\tilde{W}_O^{(l,a)} \in \mathbb{R}^{\dmodel \times d_k}$ and $\tilde{W}_Q^{(l,a)}, \tilde{W}_K^{(l,a)}, \tilde{W}_V^{(l,a)} \in \mathbb{R}^{d_k \times\dmodel}$ are the head-specific projection matrices in the multi-head attention of CSE, and $R(\pos_i)$ is the rotary matrix for the position scalar $\pos_i$. By default, we set $\pos_i = i$ (set index); when inter-set timestamps are available and informative (e.g., calendar time between encounters), $\pos_i$ can instantiate as $T_i$ to encode temporal intervals directly.

The \texttt{[CLS]} tokens act as inter-set routers: collecting set-specific features in SWE, communicating across sets in CSE, and dispatching cross-set context back to non-\texttt{[CLS]} tokens in the subsequent SWE. This cycle repeats across layers, with $\{\vc_i\}_{i=1}^m = \{\vc_i^{(L)}\}_{i=1}^m$ serving as the canonical multiset representations for downstream tasks.

\subsection{NEST Information Bypass}
\label{sec:bypass}
A core question for any layered architecture is how much input information it preserves at the final layer. The Data Processing Inequality (DPI) provides an upper bound: along any Markov chain $X \to Z \to Y$ of deterministic processing, $I(X;Y) \le I(X;Z)$, so information discarded at any intermediate stage cannot be recovered downstream. We use this principle to compare two designs for combining SWE and CSE: a \emph{sequential} design that stacks all SWE blocks before all CSE blocks, and the \emph{interleaved} design adopted by NEST. Proposition~\ref{prop:bypass} shows that the sequential design irrevocably discards token-level information at the SWE-CSE phase boundary, whereas the interleaved design preserves this information through bypass paths. See Appendix~\ref{appx:prop-bypass} for detailed proof.

\begin{proposition}[Bypass Information Preservation]
\label{prop:bypass}
Let $X_i \in \mathbb{R}^{\dmodel \times (n+1)}$ denote the embedding matrix of multiset $\mathcal{X}_i$, decomposed as the \texttt{[CLS]} embedding $\vc_i$ and $n$ event-token embeddings $Z_i \in \mathbb{R}^{\dmodel \times n}$. The sequential design forms the Markov chain
$$
X_i \xrightarrow{\,(\mathrm{SWE})^{L}\,} (\bar{\vc}_i,\, \bar{Z}_i) 
\xrightarrow[\text{(operating on \texttt{[CLS]} only)}]{\,(\mathrm{CSE})^{L}\,} \vc_i^{\text{seq}},
$$
in which $\bar{Z}_i$ is dropped at the SWE\,$\to$\,CSE phase boundary; by DPI, this discarded information is unrecoverable. The interleaved design instead forms a bypass chain in which the non-\texttt{[CLS]} states are propagated unchanged through every CSE block:
$$
X_i \xrightarrow{\,\mathrm{SWE}_{1},\,\mathrm{CSE}_{1}\,} (\vc_i^{(1)}, Z_i^{(1)}) 
\xrightarrow{\,\mathrm{SWE}_{2},\,\mathrm{CSE}_{2}\,} \cdots 
\xrightarrow{\,\mathrm{SWE}_{L},\,\mathrm{CSE}_{L}\,} (\vc_i^{\text{int}},\, Z_i^{\text{int}}).
$$
Assuming Transformer blocks SWE and CSE are realized by continuous mappings and $n \ge 1$, the interleaved output retains strictly more information about $X_i$ than the sequential counterpart:
$$
I\!\left(X_i;\, (\vc_i^{\text{int}}, Z_i^{\text{int}})\right) \;>\; I(X_i;\, \vc_i^{\text{seq}}),
$$
and the gap is precisely the conditional information $I(X_i;\, Z_i^{\text{int}} \mid \vc_i^{\text{int}}) > 0$ that the bypass preserves.
\end{proposition}

The information that the interleaved design preserves over the sequential design is precisely the residual token-level signal that does not fit into a single $\dmodel$-dimensional \texttt{[CLS]} summary. In sequential designs, all of $X_i$ must be compressed into the \texttt{[CLS]} embedding before any cross-set context arrives---a lossy compression by dimension counting (Appendix~\ref{appx:prop-bypass}, Step~2). The interleaved design defers this compression: $Z_i$ remains available at every layer for the next SWE block to revisit in light of cross-set context received via \texttt{[CLS]}. Section~\ref{sec:simulation} verifies this advantage in a controlled simulation, and Section~\ref{sec:mimic-iv-ablation} shows that removing CSE (an extreme sequential variant) consistently underperforms NEST across MIMIC-IV tasks.

\subsection{Dual-level Pretraining Procedure}
\label{sec:msm}

The interleaved architecture (Section~\ref{sec:bypass}) provides a channel for set-level information to flow through \texttt{[CLS]}, but does not enforce that \texttt{[CLS]} actually summarizes its set's contents. We close this gap with a pretraining objective that explicitly requires \texttt{[CLS]} to predict its multiset's token distribution, complementing the standard MLM objective.

\textbf{Masked Set Modeling (MSM).} Let $\mathcal{V}$ be the token vocabulary. We model each multiset $\mathcal{X}_i$ as conditionally drawn from a multinomial distribution $\text{Mult}(n;\, \pi_\theta)$ over $\mathcal{V}$, where the rate vector $\pi_\theta \in \Delta^{|\mathcal{V}|}$ depends on the surrounding event trajectory. This is a modeling simplification: it discards within-set co-occurrence structure but retains the dominant signal of which token types are present at what frequencies, which is what set-level downstream tasks consume. The MSM loss is the KL divergence between the empirical and predicted token distributions, with the predicted distribution read out from the final-layer \texttt{[CLS]} embedding $\vc_i$:
\begin{align}   
    \mathcal{L}_{\text{MSM}} &= \KL(\hat{P}_i \,\|\, \pi_\theta), \qquad \pi_\theta = \mathrm{MLP}(\vc_i), \\
    \notag &\mathrm{MLP}: \mathbb{R}^{\dmodel} \to \Delta^{|\mathcal{V}|},
\end{align}
where $\hat{P}_i$ is the empirical distribution from token-frequency counts in $\mathcal{X}_i$. We provide the full derivation from the multinomial likelihood in Appendix~\ref{appx:msm}.

\textbf{Complementarity with MLM.} MLM and MSM operate at distinct granularities and are computed in parallel during pretraining. MLM masks 20\% of individual tokens, supervising fine-grained token-level contextualization through SWE. MSM masks entire multisets with probability 40\%: within an MSM-masked set, all non-special tokens (including \texttt{[PAD]}, excluding \texttt{[CLS]}) are replaced with \texttt{[MASK]}. This full-set masking prevents information leakage from the MLM protocol~\citep{devlin2019bert, odgaard2024core}: if any non-\texttt{[CLS]} token were left visible within an MSM-masked set, SWE could reconstruct $\hat{P}_i$ locally without forcing \texttt{[CLS]} to engage cross-set context. The overall training loss is $\mathcal{L}_\text{MLM} + \mathcal{L}_\text{MSM}$. Together, these two objectives encourage \texttt{[CLS]} to aggregate set-level information through cross-set context, yielding embeddings $\{\vc_i\}$ that are directly usable for set-level downstream tasks without post-hoc pooling. Appendix~\ref{appx:msm-addon} ablates this design and confirms that MSM's contribution resides in the encoder representations rather than the task head.

\section{Experiments}
\subsection{NEST Computing Efficiency}
We compare NEST's computing efficiency with other Transformer architectures. For fairness, these counterparts are also updated using the modern recipe \citep{warner2025modernbert}. Table~\ref{tab:flops} provides the benchmark results of models' computing efficiency on randomly generated batches of \texttt{SeqSet} data. We use the same set of architectural hyperparameters for all tested models: $L = 6, \dmodel=768, d_h=2048, n_\text{heads}=12, d_{k}=64, |\mathcal{V}|=45\text{K}$, batch size of 16 and context size of 2048 ($n=32, m=64$).

All metrics are normalized to the token level. Despite a larger parameter count, NEST requires fewer FLOPs/token, achieves significantly higher throughput (paired t-test, $p < 0.05$), and consumes less peak memory. NEST's efficiency gains come from attention complexity reduction, not from being a smaller model. Moreover, unlike the approximate attention methods \citep{choromanski2020rethinking, xiong2021nystromformer} in the comparison, NEST computes \emph{exact} attention, so the efficiency gains reflect architectural effect rather than approximation error tradeoff. NEST uses Memory Efficient Attention (MEA) \citep{xFormers2022} for both SWE and CSE.

\begin{figure}[H]
\centering

\begin{minipage}[t]{0.72\linewidth}
    \vspace{0pt}
    \centering
    \captionof{table}{Benchmarking Transformer forward-pass efficiency on a single NVIDIA H200 GPU node. $N$ denotes the total context length. Longformer uses window size $w$ and $m$ global anchors; Performer uses random feature dimension size of $r = [d_k \log(d_k)]$; Nyströmformer uses $m$ landmarks; NEST uses $m$ sets with per-set size $n$.}
    \resizebox{\linewidth}{!}{
    \begin{tabular}{llcccc}
    \toprule
    model & Attn time cmplx & GFLOPs/tok $\downarrow$ & K tok/s $\uparrow$ & params (M) & Peak mem (GiB) $\downarrow$ \\
    \midrule
    BERT & $O(N^2)$ & 0.1918 & 
    $58.89 \pm 3.04$ & 77 & 52.8 \\
    $\text{BERT}_{\text{MEA}}$ \citep{xFormers2022} & $O(N^2)$ & 0.1918 &
    $65.25 \pm 4.58$ & 77 & 35.9 \\
    Longformer \citep{beltagy2020longformer} & $O(Nw + Nm)$ & 0.1714 &
    $69.11 \pm 5.05$ & 88 & 43.1 \\
    Performer \citep{choromanski2020rethinking} & $O(Nr)$ & 0.1639 &
    $77.86 \pm 5.19$ & 77 & 44.5 \\
    Nyströmformer \citep{xiong2021nystromformer} & $O(Nm)$ & \textbf{0.1575} &
    $88.79 \pm 7.26$ & 77 & 36.7 \\
    \midrule
    \textbf{NEST} & $O(Nn + m^2)$ & \textbf{0.1573} &
    $\boldsymbol{90.46 \pm 6.99}$ & 120 & \textbf{34.9} \\
    \bottomrule
    \end{tabular}
    }
    \label{tab:flops}
\end{minipage}
\hfill
\begin{minipage}[t]{0.25\linewidth}
    \vspace{0pt}
    \centering
    \includegraphics[width=\linewidth]{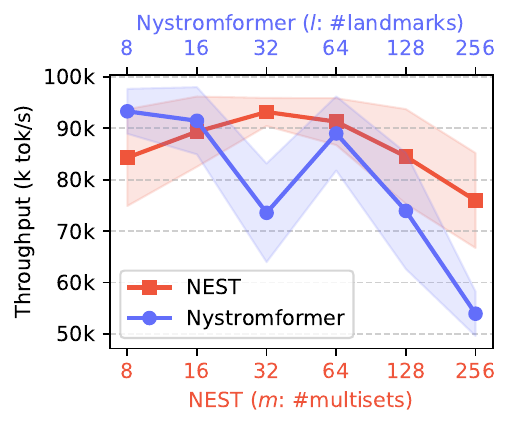}
    \vspace{-1.6em}
    \captionof{figure}{Nyströmformer suffers a sharper computational cliff at larger $l$.}
    \label{fig:nest-v-nystrom}
\end{minipage}

\end{figure}

\subsection{Advantages of NEST's Interleaving Design: A Simulation Study}
\label{sec:simulation}
Proposition \ref{prop:bypass} predicts that the interleaved design preserves strictly more information than the sequential design. However, it remains to show that this advantage translates to task performance and to characterize when it is most pronounced. Since real-world comparisons confound topology with other design choices, we construct a controlled simulation that isolates topology as the sole variable.

Clinical outcomes often depend on the interaction between individual events and a patient's broader longitudinal context — an elevated lab value may carry different clinical significance depending on the trajectory across prior visits. We define a synthetic outcome reflecting this structure: \begin{equation} y = \sum_{ij} \operatorname{softmax}_j\!\big(\langle \vx_{ij}, \boldsymbol{s}_i \rangle\big) \cdot \langle \vx_{ij}, \boldsymbol{s}_i \rangle, \end{equation} where $\boldsymbol{s}_i$ is the cross-set context summary for set $i$ and each token $\vx_{ij}$ combines a codebook content vector with Gaussian noise of controllable scale $\sigma^2$ (full Data Generation Process details in Appendix \ref{appx:simulation}). This outcome is more than set statistics alone, which requires token-level cross-set interaction.

We compare parameter-matched interleaved (SWE-CSE-$\cdots$) and sequential (SWE-$\cdots$-CSE-$\cdots$) architectures across noise scales and depths $L \in {2, 4}$. Figure \ref{fig:simulation-rlt} shows that the interleaved design consistently outperforms the sequential design, with the gap widening at greater depth, which supports Proposition \ref{prop:bypass}'s prediction that repeated bypass cycles compound the information advantage. The sequential design exhibits diminishing or even negative returns from added depth.

\begin{figure}[H]
    \centering
    \begin{subfigure}{0.495\linewidth}
        \centering
        \includegraphics[width=\linewidth]{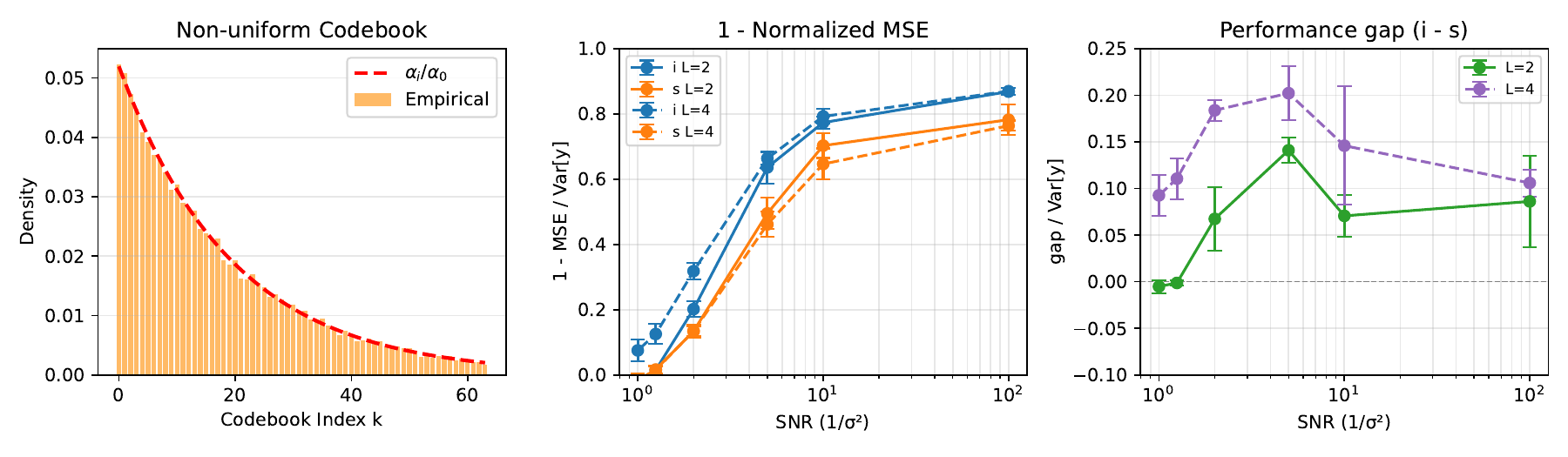}
        \vspace{-1.8em}
        \caption{Non-uniform token distribution}
    \end{subfigure}
    \hfill
    \begin{subfigure}{0.495\linewidth}
        \centering
        \includegraphics[width=\linewidth]{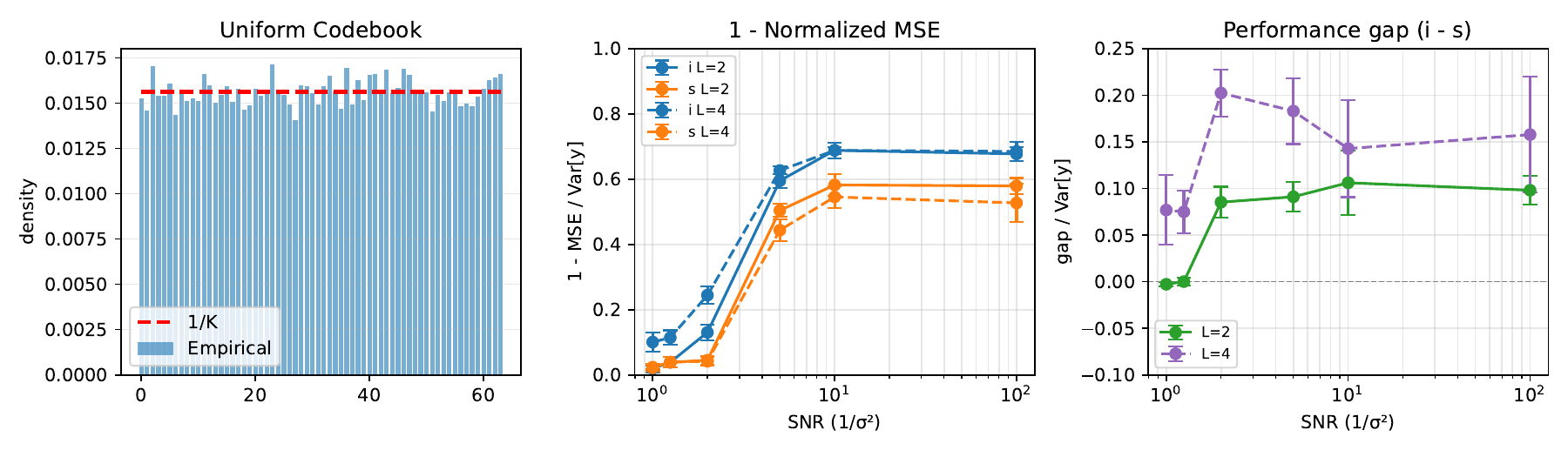}
        \vspace{-1.8em}
        \caption{Uniform token distribution}
    \end{subfigure}
    
    \caption{Comparing interleaved design with the ad-hoc sequential design.}
    \label{fig:simulation-rlt}
\end{figure}

\subsection{Real-World Data Tasks}
To ensure the reproducibility of our study, we conduct experiments on three real-world datasets, including proprietary EHR data from our institution and two publicly available datasets: MIMIC-IV-hosp and Instacart. Table~\ref{tab:data} summarizes dataset statistics.

\begin{wraptable}{r}{0.46\linewidth}
    \vspace{-0.8em}
    \caption{Real world datasets' statistics.}
    \label{tab:data}
    \centering
    \resizebox{\linewidth}{!}{
    \begin{tabular}{lcccc}
    \toprule
    Dataset     & within-set RoPE & \#subject & \#sets & \#tokens\\
    \midrule
    Instacart     & off & 206K  & 3.42M & 33.8M \\
    MIMIC-IV hosp   & on & 223K  & 546K  & 122M  \\
    Proprietary EHR & off & 500K  & 7.34M & 70.9M \\
    \bottomrule
    \end{tabular}
    }
    \vspace{-1em}
\end{wraptable}

\textbf{Instacart.}\quad
This dataset is open sourced by Instacart\footnote{While the official source is no longer directly hosted, a copy of the dataset is available on Kaggle: \textit{instacart-online-grocery-basket-analysis-dataset}.}. The dataset consists of online grocery transactions used to develop model for Next Basket Recommendation (NBR). Orders are treated as baskets containing unordered and equally important items \citep{li2023next}, so we turn off the within-set RoPE.

\textbf{MIMIC-IV \texttt{hosp}.}\quad
MIMIC-IV is a public dataset comprising de-identified electronic health records from Beth Israel Deaconess Medical Center \citep{johnson2023mimic}. We use version 3.1, which spans admissions from 2008 to 2022. We extract data from the \texttt{hosp} module across six source tables: \texttt{admissions}, \texttt{diagnoses\_icd}, \texttt{procedures\_icd}, \texttt{labevents}, \texttt{emar}, and \texttt{prescriptions}. The processed dataset contains 546K hospital admissions from 223K patients. We adopt the ETHOS's \citep{renc2024ethos} tokenization protocol with slight adjustments for the MIMIC-IV clinical event streams. 

\textbf{Proprietary institutional EHR.}\quad
We developed an ETL pipeline that transforms heterogeneous tables from our institutional EHR database into a unified event-stream format. We collected medical event stream data from approximately 0.5M patients born between 1999 and 2024 and filtered events occurring between ages 0 and 18. Modeling the long-term patient timelines, we ignore the order of events within each encounter. We assign patients born on the 17th–24th of each month to the test set and the remainder to the training and validation splits.

\subsubsection{Next-Basket Recommendation} \label{sec:instacart}

Given a user's purchase history $\big((\mathcal{X}_i, T_i)\big)_{i=1}^{m}$, next basket recommendation (NBR) predicts the set of items $\mathcal{X}_{m+1}$ the user will purchase next. Instacart exhibits a high repeated purchase rate~\citep{li2023next}, making repeat-aware methods, such as Sets2Sets~\citep{hu2019sets2sets}, TIFUKNN~\citep{hu2020tifuknn}, and ReCANet~\citep{ariannezhad2022recanet}, particularly effective. We use this task to evaluate whether NEST's pretrained set-level representations transfer to downstream applications.

\begin{wraptable}{r}{0.3\linewidth}
    \vspace{-0.8em}
    \caption{Comparison on NBR.}
    \label{tab:nbr}
    \centering
    \resizebox{\linewidth}{!}{
    \begin{tabular}{lcc}
    \toprule
    Model & Recall@10 & NDCG@10 \\
    \midrule

    \multicolumn{3}{l}{\textit{Repeat-aware}} \\
    Sets2Sets \citep{hu2019sets2sets}   & 21.25 & 19.23 \\
    TIFUKNN \citep{hu2020tifuknn}       & 34.97 & 35.70 \\
    ReCANet \citep{ariannezhad2022recanet} & \textbf{35.41} & 36.01 \\
    \midrule

    \multicolumn{3}{l}{\textit{Pretrained FM}} \\
    BERT4Rec \citep{sun2019bert4rec, li2023btbr} & 24.83 & 30.42 \\
    NEST & 25.96 & 35.95 \\
    \midrule

    \multicolumn{3}{l}{\textit{FM + Repeat-aware}} \\
    BERT + ReCANet & 34.60 & 42.74 \\
    NEST + ReCANet & 34.77 & $\mathbf{43.01}$ \\
    
    \bottomrule
    \end{tabular}
    }
    \vspace{-1em}
\end{wraptable}

We consider two settings. First, following the BTBR inference protocol~\citep{li2023btbr}, we append a query basket $\mathcal{X}_{m+1} = \{\texttt{[MASK]}\}$ and rank items by logit at the mask position (no fine-tuning). Second, we replace ReCANet's randomly initialized embeddings with NEST's pretrained representations: NEST's embedding layer serves as the item embedding, and NEST's contextualized \texttt{[CLS]} representations serve as the user embedding, evaluating their quality as a feature backbone for a task-specific model.

Table~\ref{tab:nbr} shows that NEST alone is competitive with BERT4Rec but underperforms state-of-the-art repeat-aware baselines, as expected given that its pretraining objective does not explicitly model repeat purchases. However, NEST + ReCANet outperforms all baselines on NDCG@10, indicating that NEST's pretrained representations capture richer user-item dynamics than random initialization.

\subsubsection{MIMIC-IV}

We compare NEST against five baselines covering major EHR FM paradigms: Hi-BEHRT~\citep{li2022hibehrt}, CEHR-BERT~\citep{pang2021cehr}, MOTOR~\citep{steinberg2023motor}, CORE-BEHRT~\citep{odgaard2024core}, and GT-BEHRT~\citep{poulain2024graph}, all built with matched architecture hyperparameters for fair comparison. We originally planned to include HEART~\citep{huang2024heart}, but its $O(\sum_i n_i^2)$ pairwise edge embeddings caused frequent OOM failures during training, so we excluded it from the comparison. The evaluation covers five downstream tasks: inpatient mortality prediction, 30-day readmission, prolonged length of stay, ICD chapter classification, and ICD category multi-label classification. Detailed task definitions are provided in Appendix~\ref{appx:downstream}.

\begin{table}[H]
    \centering
    \small
    \caption{MIMIC-IV finetuned downstream performance (mean $\pm$ SD across 5 random seeds).}
    \begin{subtable}[t]{\linewidth}
        \centering
        \caption{AUROC}
        \resizebox{0.8\linewidth}{!}{
        \begin{tabular}{lccccc}
        \toprule
        Model & Mortality & Readmission & Length of Stay & ICD Chapter & ICD Category \\
        \midrule
        Hi-BEHRT & $89.00 \pm 0.53$ & $68.79 \pm 0.30$ & $82.03 \pm 0.42$ & $89.18 \pm 1.46$ & $88.64 \pm 0.26$ \\
        MOTOR & $90.60 \pm 0.47$ & $69.05 \pm 0.55$ & $85.62 \pm 0.81$ & $91.42 \pm 0.71$ & $92.35 \pm 0.16$ \\
        CEHR-BERT & $93.09 \pm 0.20$ & $70.59 \pm 0.38$ & $85.23 \pm 0.10$ & $93.03 \pm 0.56$ & $92.03 \pm 0.25$ \\
        CORE-BEHRT & $92.96 \pm 0.36$ & $71.02 \pm 0.20$ & $86.23 \pm 0.15$ & $93.51 \pm 0.43$ & $92.25 \pm 0.05$ \\
        GT-BEHRT & $93.80 \pm 0.11$ & $70.82 \pm 0.31$ & $84.84 \pm 0.15$ & $91.31 \pm 0.94$ & $91.98 \pm 0.05$ \\
        \midrule
        NEST$_\text{MLM Only}$ (Ours) & $95.28 \pm 0.09$ & $72.27 \pm 0.36$ & $88.50 \pm 0.08$ & $94.81 \pm 0.64$ & $92.59 \pm 0.25$ \\
        NEST (Ours) & $\mathbf{95.57 \pm 0.10}$ & $\mathbf{72.53 \pm 0.35}$ & $\mathbf{88.79 \pm 0.14}$ & $\mathbf{95.29 \pm 0.46}$ & $\mathbf{93.51 \pm 0.19}$ \\
        \bottomrule
        \end{tabular}
        }
    \end{subtable}
    \begin{subtable}[t]{0.8\linewidth}
        \centering
        \caption{AP}
        \resizebox{\linewidth}{!}{
        \begin{tabular}{lccccc}
        \toprule
        Model & Mortality & Readmission & Length of Stay & ICD Chapter & ICD Category \\
        \midrule
        Hi-BEHRT & $32.64 \pm 0.94$ & $49.31 \pm 0.34$ & $48.50 \pm 0.68$ & $43.93 \pm 3.85$ & $10.81 \pm 0.17$ \\
        MOTOR & $36.67 \pm 1.46$ & $53.87 \pm 0.61$ & $58.28 \pm 0.59$ & $49.05 \pm 1.30$ & $15.24 \pm 0.18$ \\
        CEHR-BERT & $42.52 \pm 0.74$ & $55.62 \pm 0.57$ & $57.56 \pm 0.21$ & $52.95 \pm 2.71$ & $15.49 \pm 0.11$ \\
        CORE-BEHRT & $45.12 \pm 2.16$ & $56.10 \pm 0.51$ & $58.88 \pm 0.39$ & $55.85 \pm 2.00$ & $15.53 \pm 0.26$ \\
        GT-BEHRT & $41.58 \pm 1.00$ & $56.39 \pm 0.52$ & $55.87 \pm 0.17$ & $53.73 \pm 1.42$ & $12.58 \pm 0.23$ \\
        \midrule
        NEST$_\text{MLM only}$ (Ours) & $51.19 \pm 1.50$ & $58.32 \pm 0.44$ & $64.04 \pm 0.22$ & $60.18 \pm 1.96$ & $16.38 \pm 0.12$ \\
        NEST (Ours) & $\mathbf{52.75 \pm 1.08}$ & $\mathbf{58.68 \pm 0.53}$ & $\mathbf{64.58 \pm 0.49}$ & $\mathbf{60.85 \pm 1.90}$ & $\mathbf{16.92 \pm 0.15}$ \\
        \bottomrule
        \end{tabular}
        }
    \end{subtable}

\end{table}

\subsubsection{Proprietary EHR}

In downstream tasks, we obtain patient representation by encoding the truncated prefix $\big((\mathcal{X}_i, T_i)\big)_{\le a}$ to predict targeted clinical events occurring after the discharge time $T_a$. This calculation examines FMs' performances on preventive risk prediction.

We curated three binary classification cohorts: (i) Recurrent Acute Otitis Media (rAOM) ($22\%$), (ii) Autism Spectrum Disorder (ASD) ($2.1\%$), (iii) Home Health Care 30-day readmission ($20\%$). More details of the data pipeline and cohorts' clinical significance are provided in Appendix~\ref{appx:proprietary-ehr}. These tasks focus on pediatric cohorts, aiming to support in-house clinical decision-making for preventive care. We benchmark our model's performance on these tasks against count-based LightGBM \citep{ke2017lightgbm, wornow2023ehrshot} and CORE-BEHRT \cite{odgaard2024core}. 

In Table \ref{tab:private-ehr-ds}, for foundation models, we report performance under both nonlinear probing (with the backbone frozen) and full fine-tuning (with the backbone adapted to the downstream task).

\begin{table}[H]
    \caption{Performance on proprietary EHR downstream tasks.}
    \label{tab:private-ehr-ds}
    \centering
    \resizebox{0.65\linewidth}{!}{
    \begin{tabular}{lcccccc}
    \toprule
    Model 
    & \multicolumn{2}{c}{rAOM}
    & \multicolumn{2}{c}{ASD}
    & \multicolumn{2}{c}{30d Readmit} \\
    \cmidrule(lr){2-3}
    \cmidrule(lr){4-5}
    \cmidrule(lr){6-7}
    & AUROC & AP & AUROC & AP & AUROC & AP \\
    \midrule
    LightGBM 
        & 71.7 & 40.1 & 73.6 & 6.0 & 73.8 & 46.6 \\
    CORE-BEHRT$_{\texttt{pb}}$
        & 72.1 & 39.6 & 70.3 & 9.8 & 73.8 & 46.8 \\
    CORE-BEHRT$_{\texttt{ft}}$
        & 73.0 & \textbf{43.9} & 73.8 & 10.2 & 73.8 & 46.8 \\
    \midrule
    \text{NEST}$_{\texttt{pb}}$ 
        & 73.7 & 42.6 & 71.1 & 7.6 & 73.4 & 45.7 \\
    \text{NEST}$_{\texttt{ft}}$ 
        & \textbf{73.8} & 43.2 & \textbf{74.2} & \textbf{10.9} & \textbf{74.1} & \textbf{47.5} \\
    \bottomrule
    \end{tabular}
    }
\end{table}

\section{Ablation Studies}
\label{sec:mimic-iv-ablation}

We probe NEST's design through two ablation regimes on three MIMIC-IV tasks (mortality, 30-day readmission, prolonged length of stay), averaged over 5 random seeds. \textbf{Group B (architectural)} weakens internal components and re-pretrains from scratch; \textbf{Group A (behavioral)} keeps the NEST (MLM+MSM) checkpoint fixed and intervenes only at fine-tuning. Full setup is in Appendix~\ref{appx:ablation-setup}. Table~\ref{tab:ablation} summarizes the results.

\begin{table}[H]
    \centering
    \small
    \caption{Ablation study on MIMIC-IV (mean $\pm$ SD across 5 random seeds; AUROC / AP, in \%). Architectural ablations (Group B) require re-pretraining; behavioral ablations (Group A) reuse the NEST (MLM+MSM) checkpoint and only modify fine-tuning. Variants marked with $\dagger$ are parameter-matched to NEST ($\sim$120M).}
    \label{tab:ablation}
    \resizebox{\linewidth}{!}{
    \begin{tabular}{llcccccc}
    \toprule
    & & \multicolumn{2}{c}{Mortality} & \multicolumn{2}{c}{Readmission} & \multicolumn{2}{c}{Length of Stay} \\
    \cmidrule(lr){3-4} \cmidrule(lr){5-6} \cmidrule(lr){7-8}
    & Variant & AUROC & AP & AUROC & AP & AUROC & AP \\
    \midrule
    \multirow{4}{*}{\shortstack[l]{Group B\\(re-pretrained)}}
        & SWE\,$\to$\,MeanPool (6L)                & $92.18 \pm 0.39$ & $41.96 \pm 0.80$ & $68.56 \pm 0.26$ & $53.24 \pm 0.88$ & $85.38 \pm 0.18$ & $56.13 \pm 0.44$ \\
        & $-$\,CSE (6L SWE-only)                   & $92.85 \pm 0.15$ & $47.32 \pm 1.00$ & $70.37 \pm 0.26$ & $56.05 \pm 0.23$ & $86.24 \pm 0.11$ & $60.86 \pm 0.34$ \\
        & SWE\,$\to$\,MeanPool (12L)$^\dagger$     & $93.05 \pm 0.31$ & $45.58 \pm 0.89$ & $69.26 \pm 0.23$ & $55.30 \pm 0.26$ & $86.84 \pm 0.23$ & $59.22 \pm 0.75$ \\
        & $-$\,CSE (12L SWE-only)$^\dagger$        & $94.52 \pm 0.19$ & $48.84 \pm 1.17$ & $71.45 \pm 0.24$ & $57.33 \pm 0.32$ & $87.71 \pm 0.10$ & $62.35 \pm 0.25$ \\
    \midrule
    \multirow{5}{*}{\shortstack[l]{Group A\\(fine-tune only)}}
        & Linear probe                             & $94.33 \pm 0.18$ & $46.65 \pm 1.07$ & $70.40 \pm 0.20$ & $55.60 \pm 0.39$ & $86.10 \pm 0.20$ & $60.30 \pm 0.35$ \\
        & Temporal shuffle                         & $94.47 \pm 0.20$ & $48.98 \pm 0.79$ & $70.78 \pm 0.29$ & $56.16 \pm 0.43$ & $87.65 \pm 0.05$ & $62.41 \pm 0.40$ \\
        & \texttt{[CLS]} $\to$ mean pool           & $94.59 \pm 0.23$ & $48.18 \pm 0.85$ & $71.93 \pm 0.16$ & $57.87 \pm 0.37$ & $88.38 \pm 0.13$ & $62.95 \pm 0.42$ \\
        & Freeze SWE                               & $94.43 \pm 0.24$ & $47.98 \pm 1.14$ & $71.25 \pm 0.22$ & $57.09 \pm 0.48$ & $87.73 \pm 0.15$ & $61.71 \pm 0.43$ \\
        & Freeze CSE                               & $94.99 \pm 0.21$ & $50.53 \pm 1.15$ & $71.91 \pm 0.18$ & $57.99 \pm 0.34$ & $88.21 \pm 0.14$ & $63.45 \pm 0.54$ \\
    \midrule
    & \textbf{NEST (full)}                         & $\mathbf{95.57 \pm 0.10}$ & $\mathbf{52.75 \pm 1.08}$ & $\mathbf{72.53 \pm 0.35}$ & $\mathbf{58.68 \pm 0.53}$ & $\mathbf{88.79 \pm 0.14}$ & $\mathbf{64.58 \pm 0.49}$ \\
    \bottomrule
    \end{tabular}
    }
\end{table}

Both encoders are essential, and depth cannot substitute for either. CSE's contribution comes from temporal structure, not from cross-set attention alone. The \texttt{[CLS]} pathway carries set-level information that mean pooling cannot recover. SWE representations require more task-specific adaptation than CSE representations. Together, these ablations show that NEST's gains arise from its architectural commitments rather than capacity or pretraining artifacts. See Appendix~\ref{appx:mimic-abl} for detailed discussion.

\section{Conclusion}
We introduced NEST, a hierarchical Transformer for sequence-of-multiset (\texttt{SeqSet}) event streams. By interleaving within-set (SWE) and cross-set (CSE) encoders, NEST preserves multiset boundaries and induces structured sparsity, achieving higher throughput than BERT-style baselines while exposing a token-level bypass path that prevents the information loss inherent in sequential SWE-then-CSE designs (Proposition~\ref{prop:bypass}). Masked Set Modeling adds explicit set-level supervision, yielding directly usable \texttt{[CLS]} embeddings without post-hoc pooling. Across MIMIC-IV, a proprietary pediatric EHR, and Instacart NBR, NEST consistently outperforms strong baselines; ablations confirm that gains arise from the interleaved architecture and MSM rather than capacity. NEST is a step toward foundation models that respect, rather than flatten, the structure of real-world event streams.

\textbf{Limitations.} NEST is scoped to representation-learning foundation models, since our goal is patient representation quality rather than benchmarking all model families. Generative trajectory models~\citep{renc2024ethos, mcdermott2023event} optimize a different objective and require a separate evaluation protocol, which we leave to future work. We report discriminative metrics (AUROC/AP) but do not assess calibration, which is essential for clinical risk estimation. Although SeqSet extends beyond healthcare, our cross-domain evidence is limited to a single non-EHR dataset (Instacart); broader validation across event-stream domains (e.g., financial transactions, user-behavior logs) remains an important direction.

\textbf{Broader Impact.} The use of ML in healthcare carries significant ethical responsibilities. To ensure data privacy, we exclusively used de-identified MIMIC-IV data accessed via PhysioNet credentialing and IRB-approved proprietary EHRs within secure environments. We acknowledge the risk of algorithmic bias, particularly for under-documented populations due to historical EHR disparities; rigorous fairness evaluations across demographic subgroups are essential prior to any real-world deployment. To mitigate risks from uncertainty and overconfidence, our model is strictly designed as an assistive Clinical Decision Support tool requiring human-in-the-loop oversight. Future deployment would benefit from interpretability analyses such as attention visualization to support clinician trust.

\newpage
{
\small
\bibliographystyle{plainnat}

}

\newpage
\appendix
\section{Proposition 1: Bypass Information Preservation}
\label{appx:prop-bypass}

\begin{figure}[H]
    \centering
    \includegraphics[width=0.5\linewidth]{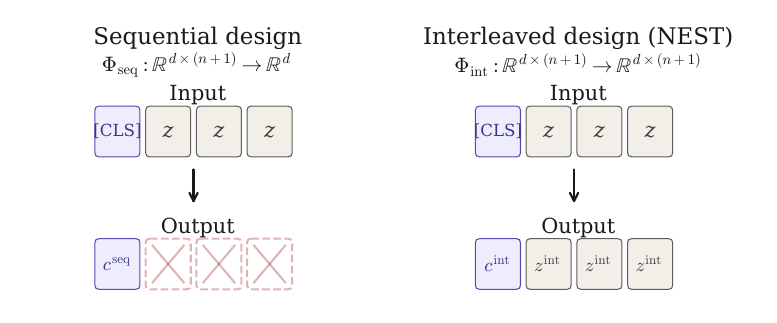}
    \caption{Interleave design preserve token-level information in each NEST layer.}
    \label{fig:i-vs-s}
\end{figure}

\begin{proof}
We prove the inequality in two steps.

\emph{Step 1 (chain rule of mutual information).} 
$$
I\!\left(X_i;\, (\vc_i^{\text{int}}, Z_i^{\text{int}})\right) 
\;=\; I(X_i;\, \vc_i^{\text{int}}) \;+\; I(X_i;\, Z_i^{\text{int}} \mid \vc_i^{\text{int}}).
$$
The interleaved design subsumes the sequential one in expressive capacity at matched parameter and depth budgets: fixing all cross-set attention weights to zero in interleaved CSE blocks recovers a pure SWE pipeline, after which selectively activating CSE blocks reproduces the sequential output. Hence $I(X_i;\, \vc_i^{\text{int}}) \ge I(X_i;\, \vc_i^{\text{seq}})$.

\emph{Step 2 (strict bypass gap).} 
Suppose, for contradiction, that $I(X_i;\, Z_i^{\text{int}} \mid \vc_i^{\text{int}}) = 0$. Then $Z_i^{\text{int}}$ is determined by $\vc_i^{\text{int}}$ almost surely: there exists a measurable map $\phi: \mathbb{R}^{\dmodel} \to \mathbb{R}^{\dmodel \times n}$ such that $Z_i^{\text{int}} = \phi(\vc_i^{\text{int}})$. The image of the interleaved encoding $\Psi: X_i \mapsto (\vc_i^{\text{int}}, Z_i^{\text{int}})$ would therefore lie in the graph $\{(\vc, \phi(\vc)) : \vc \in \mathbb{R}^{\dmodel}\} \subset \mathbb{R}^{\dmodel \times (n+1)}$, a subset of topological dimension at most $\dmodel$. But the input domain $\mathbb{R}^{\dmodel \times (n+1)}$ has dimension $\dmodel(n+1) > \dmodel$ when $n \ge 1$, so by the topological invariance of dimension, $\Psi$ would have to collapse a positive-dimensional set of inputs onto its image---a degeneracy precluded by standard Transformer blocks, whose residual connections prevent global volume collapse. Hence $I(X_i;\, Z_i^{\text{int}} \mid \vc_i^{\text{int}}) > 0$.

Combining the two steps,
$$
I\!\left(X_i;\, (\vc_i^{\text{int}}, Z_i^{\text{int}})\right) \;\ge\; I(X_i;\, \vc_i^{\text{seq}}) \;+\; I(X_i;\, Z_i^{\text{int}} \mid \vc_i^{\text{int}}) \;>\; I(X_i;\, \vc_i^{\text{seq}}).
$$
\end{proof}

\section{MSM}
\label{appx:msm}
\subsection{Derivation}
\newcommand{\argmax}{\text{argmax}}
\newcommand{\argmin}{\text{argmin}}

Consider a family of probability measures $\{\pi_\theta : \theta\in\Theta\}$ on the finite space $\Omega=\{\omega_1,\ldots,\omega_k\}$. The KL divergence of $\pi_\theta$ from a fixed measure $P$ decomposes as $$ \KL[P; \pi_\theta] = -H(P) - \sum_{i=1}^k P[\omega_i]\log \pi_\theta[\omega_i], $$ so since $H(P)$ is constant in $\theta$, $$ \argmin_{\theta\in\Theta}\KL[P;\pi_\theta] = \argmax_{\theta\in\Theta}\sum_{i=1}^k P[\omega_i]\log \pi_\theta[\omega_i]. $$ That is, minimizing KL divergence is equivalent to maximizing the negative cross-entropy.

Applying this to the multinomial: let $\mathbf{x}=(x_1,\ldots,x_k)$ be observed counts under $\mathbf{X}\sim\mathrm{Multinomial}(n,\pi_\theta)$, and define the empirical measure $\hat{P}_i = x_i/n$. Since the log-likelihood satisfies $$ \ell(\theta;\mathbf{x}) = C + \sum_{i=1}^k x_i\log\theta_i, \quad C \text{ free of } \theta, $$ we have $$ \argmax_{\theta\in\Theta}\ell(\theta;\mathbf{x}) = \argmax_{\theta\in\Theta}\sum_{i=1}^k \hat{P}_i\log\pi_\theta(i) = \argmin_{\theta\in\Theta}\KL[\hat{P}; \pi_\theta]. $$ Hence MLE is equivalent to minimizing the KL divergence from the empirical measure to the model family.

\subsection{Forced Cross-Set Aggregation}

The MSM masking protocol (Section~\ref{sec:msm}) forces \texttt{[CLS]} to encode set-level structure as a direct consequence of the objective, rather than as an emergent side-effect of self-attention. Concretely, with every intra-set token masked under MSM, no within-set signal is available to reconstruct $\hat{P}_i$. The \texttt{[CLS]} token must instead aggregate cross-set context through CSE---inferring what tokens belong to a clinical encounter from the patient's broader trajectory. As a result, NEST's \texttt{[CLS]} embeddings $\{\vc_i\}$ are usable directly for set-level downstream tasks without post-hoc pooling, distinguishing NEST from contrastive methods~\citep{gao2021simcse, chuang2022diffcse} that provide only implicit pressure for the token to summarize set contents.

\subsection{Ablation Study: What Does MSM Provide?}
\label{appx:msm-addon}
We ablate the MSM objective to understand its contribution to NEST pretraining. Specifically, we ask: (1) Does MSM improve set-level prediction? (2) Does the learned representation transfer beyond the task-specific head?

We pretrain three small (3-layer) NEST variants on the proprietary EHR data for 10 epochs: MSM-only, MLM-only, and MLM+MSM. To evaluate whether set-level knowledge resides in the encoder rather than the MSM task head, we measure set prediction performance using a weight-tied linear layer~\citep{press2017weighttying} instead of the MSM task head used in the pretraining.

\begin{table}[H]
    \caption{Set prediction performance (Recall@10, NDCG@10) on randomly masked encounters.}
    \label{tab:mlm-vs-msm}
    \centering
    \begin{tabular}{lcc}
    \toprule
    Model & Recall@10 & NDCG@10 \\
    \midrule
    MSM-only  & $<$0.01 & $<$0.01 \\
    MLM-only  & 47.6 & 51.6 \\
    MLM+MSM   & \textbf{48.9} {\small($\uparrow$2.7\%)} & \textbf{54.9} {\small($\uparrow$6.4\%)} \\
    \bottomrule
    \end{tabular}
\vskip -0.1in
\end{table}

MSM-only pretraining fails entirely: the task head learns to predict sets, but the encoder representations do not transfer. MLM-only provides a strong baseline, but combining MLM with MSM yields consistent improvements across both metrics. This confirms that the two objectives are complementary: MLM drives token-level contextualization, while MSM encourages the \texttt{[CLS]} embedding to capture set-level information in a transferable manner.

\section{Multiset Prediction}
\label{appx:set-pred}
During pretraining, we employ a separate MSM task head for set prediction. However, this setup may overestimate representation quality, as set-level information can be largely absorbed by the MSM-task head while the FM backbone remains insufficiently optimized—in other words, the MSM-task head may perform most of the heavy lifting. To mitigate this issue, we adopt a more conservative strategy when evaluating FM representation quality by computing the missing tokens for a masked set without relying on the MSM head. Specifically, analogous to MLM, we pass the output representation through the shared, weight-tied embedding projection to compute item logits, then select the top-K scoring items to form the predicted basket.

As the predicted set is formed from the top-K predicted token logits, we evaluate set prediction using Recall@K and NDCG@K (normalized discounted cumulative gain). Recall@K measures the fraction of ground-truth basket items recovered within the top-K predictions, while NDCG@K rewards models that rank ground-truth items more confidently toward the top of that list of logits.
\begin{align}
    \text{Recall}@K(u) &= \frac{\left|T_{u} \cap P_{u} \right|}{\left|T_{u} \right|} \\
    \text{NDCG}@K(u) &= \frac{\sum_{k=1}^K p_k / \log_2(k+1)}{\sum_{k=1}^{\min\left(K, \left|T_{u}\right|\right)} 1/\log_2(k+1)}
\end{align}
where $T_u$ and $P_u$ are the ground truth set (which need not be in a size of $K$) and predicted set with $K$ recommended items for subject $u$ and $p_k = \mathbb{I}\{P_{u}^{(k)} \in T_{u}\}$. Recall$@K$ quantifies model's ability to find the relevant items, and NDCG$@K$ measures model's ability to push all correct tokens to the front of the recalled tokens.

\section{Simulation Studies}
\subsection{Computing Efficiency}

We calculate the throughput over 100 random batches across five random seeds. We reported pooled standard errors.
$$
\sigma_{\text{pooled}} = \sqrt{\frac{1}{n}\sum_{i=1}^{n} s_i^2 + \text{Var}(\bar{x}_i)}
$$

\subsection{Interleave}
\label{appx:simulation}
Let $\mathcal{C} \in \mathbb{R}^{K\times \dmodel}$ be a token content codebook with each row drawn i.i.d.\ from $\mathcal{N}(0,I_{\dmodel})$. For each set $i$, draw $\boldsymbol{\pi}^{(i)} \sim \mathrm{Dir}(\boldsymbol{\alpha})$, sample $k_{ij}|\boldsymbol{\pi} \overset{\text{i.i.d.}}{\sim} \mathrm{Cat}(\boldsymbol{\pi}^{(i)})$, and set
\[
\vx_{ij} = \mathcal{C}[k_{ij}] + \boldsymbol{\epsilon}_{ij}, \quad \boldsymbol{\epsilon}_{ij} \sim \mathcal{N}(0, \sigma^2 I_{d_{\text{model}}}),
\]
so each token combines codebook content with Gaussian noise. Then, we define the cross-set summary and synergy-dependent outcome as follows:
\begin{equation}
\boldsymbol{s}_i = \frac{1}{(T-1)\sqrt{n}} \sum_{i' \ne i} \sum_j \vx_{i'j},
\qquad
y = \sum_{ij} \operatorname{softmax}_j\!\big(\langle \vx_{ij}, \boldsymbol{s}_i \rangle\big) \cdot \langle \vx_{ij}, \boldsymbol{s}_i \rangle.
\end{equation}
where the softmax scores are driven both by the token content but also noise with controllable scales.
In the simulation experiments, we use two different kinds of $\boldsymbol\alpha$: (a) $\boldsymbol\alpha = (\alpha_0,...,\alpha_0)$ that results in uniform token distribution, and (b). $\boldsymbol\alpha = (r\alpha,r\alpha^{2},...,r\alpha^{K})$ that results in nonuniform token distribution. We also conduct experiments under varied level of noise by changing $\sigma^2$.
In this simulation study, we evaluated two distinct configurations of the SWE and CSE blocks: an interleaved arrangement (\texttt{SWE-CSE-SWE-CSE}), which is the NEST architecture, and an ad hoc sequential arrangement (\texttt{SWE-SWE-CSE-CSE}).

The simulation uses end-to-end supervised training without self-supervised pretraining, so we adopt Pooling by Multihead Attention~\citep{lee2019set} for set-level pooling as a proof of concept, which is architecturally guaranteed to aggregate token information.

\section{Real World Datasets}

\subsection{Instacart}

\begin{figure}[H]
    \centering
    \includegraphics[width=0.7\linewidth]{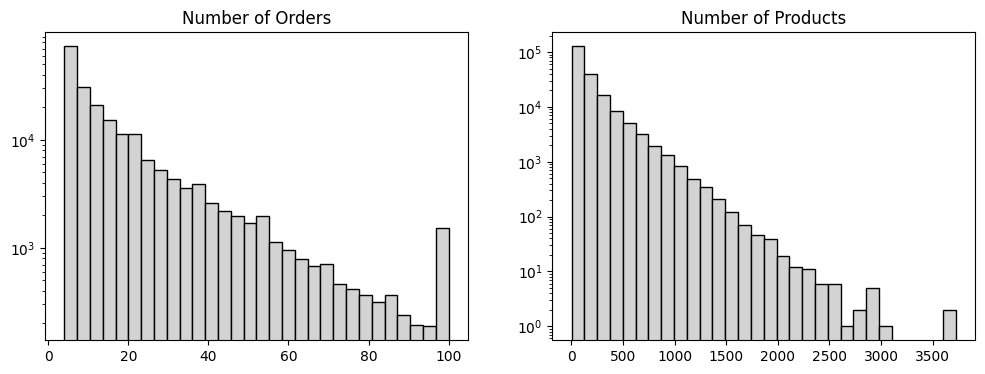}
    \caption{Distributions for \# multisets and \# tokens per user.}
    \label{fig:instacart-dist}
\end{figure}

Instead of using a random split, we follow the official split. For each user in the training set, the chronological history of baskets is used for training. For final inference evaluation, the last basket in the test set is used as the outcome, while the preceding basket history in the test set is used as the input.

\subsection{MIMIC-IV v3.1 - hosp}
\label{appx:mimic}

\begin{figure}[H]
    \centering
    \includegraphics[width=0.7\linewidth]{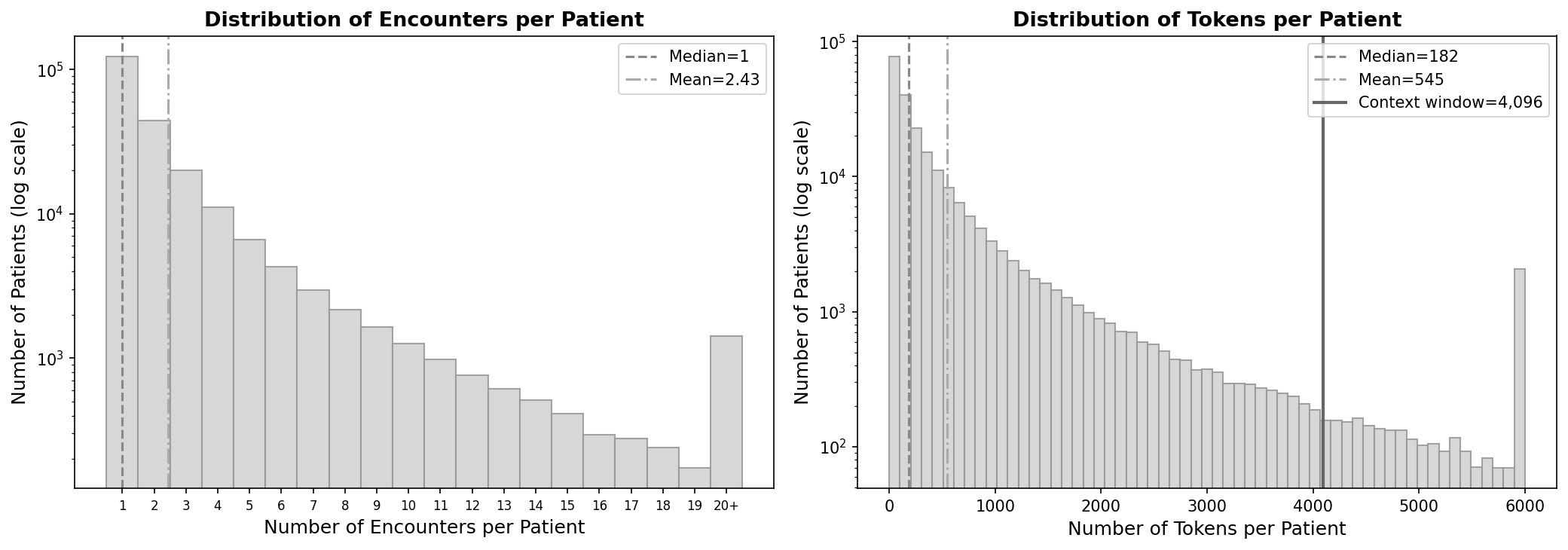}
    \caption{Distributions for \# multisets and \# tokens per patient.}
    \label{fig:mimic-dist}
\end{figure}

\paragraph{Tokenization.} Following ETHOS \citep{renc2024ethos}, we construct four clinical event types with standardized medical codes:

\begin{itemize}[topsep=2pt, itemsep=2pt, parsep=0pt, partopsep=0pt]
    \item \textbf{Diagnoses (\texttt{DX}):} ICD-9-CM codes are converted to ICD-10-CM using CMS General Equivalence Mappings (GEM). When one ICD-9 maps to multiple ICD-10 codes, we select the shortest following ETHOS. Unmappable codes (marked \texttt{NoDx} in GEM) are excluded. Format: \texttt{DX:\{icd10\}}.
    
    \item \textbf{Procedures (\texttt{PR}):} ICD-9-PCS codes are converted to ICD-10-PCS with the same shortest-match strategy. Unmappable codes (\texttt{NoPCS}) are excluded. Format: \texttt{PR:\{icd10pcs\}}.
    
    \item \textbf{Medications (\texttt{MED}):} We extract administered medications from \texttt{emar} (filtering \texttt{event\_txt='Administered'}). Drug names are mapped to ATC codes via a lookup table constructed from \texttt{prescriptions}: the GSN (Generic Sequence Number) field is mapped to ATC using external GSN-ATC mappings. Format: \texttt{MED:\{atc\}}.
    
    \item \textbf{Labs (\texttt{LAB}):} We select the 200 most frequent lab items, covering approximately 95\% of all lab observations. Continuous values are discretized into deciles (Q1--Q10) using population-level quantile thresholds. Labs occurring within 24 hours before admission are also included. Format: \texttt{LAB:\{itemid\}\_Q\{1-10\}}. This approach integrates numerical information into the modeling process.

\end{itemize}

\paragraph{Hierarchical structure.}
The principal NEST-specific adaptation is the \emph{hierarchical SeqSet} formulation: rather than ETHOS's flat 1D timeline with discrete time-interval separator tokens, NEST groups events into a sequence of multisets, where each multiset corresponds to a single hospital admission. Within each multiset, diagnoses---which lack explicit timestamps in MIMIC-IV---are anchored at the admission's start, following ETHOS's convention of using \texttt{admittime} as the diagnosis timestamp; procedures, medications, and lab events are ordered chronologically by their recorded times (\texttt{chartdate} for procedures, \texttt{charttime} for medications and labs). Across multisets, \texttt{days\_since\_prior\_admission} instantiates the inter-set position scalar $\pos_i$ (Section~\ref{sec:dual-contextualize}) for CSE's RoPE.

\subsubsection{Downstream Tasks}
\label{appx:downstream}

We evaluate NEST on five downstream tasks, covering clinical outcome prediction, diagnosis classification, and future event prediction. All tasks use patient-level 70/15/15 train/validation/test splits to prevent information leakage between patients.

\paragraph{Prediction Mechanism.} For all downstream tasks, the pre-trained encoder processes the patient's admission sequence and produces contextualized representations. Each segment is prepended with a \texttt{[CLS]} token. For classification, we extract the \texttt{[CLS]} representation from the last valid segment and pass it through a two-layer MLP (with GELU activation and dropout) to produce logits. Binary and multi-class tasks use cross-entropy loss; multi-label tasks (ICD category) use binary cross-entropy loss.

\paragraph{Task Definitions.}

\begin{itemize}[topsep=2pt, itemsep=4pt, parsep=0pt, partopsep=0pt]

\item \textbf{Inpatient Mortality.} Binary classification predicting whether a patient dies during the hospital stay. Following prior work, we use only data from the first 24 hours after admission to enable early prediction while preventing data leakage. Diagnoses from the target admission are excluded. Evaluated using AUROC and AUPRC.

\item \textbf{30-Day Readmission.} Binary classification predicting whether a patient is readmitted within 30 days after discharge. Patients who die during hospitalization or have no subsequent admission record are excluded from this task. Prediction is made at discharge time using full admission data (including diagnoses). Evaluated using AUROC and AUPRC.

\item \textbf{Prolonged Length of Stay.} Binary classification predicting whether the hospital stay exceeds 7 days. We use only data from the first 48 hours after admission for prediction. Diagnoses from the target admission are excluded to prevent leakage. Evaluated using AUROC and AUPRC.

\item \textbf{ICD Chapter Classification.} Multi-class classification predicting the primary diagnosis ICD chapter (based on the \texttt{seq\_num=1} diagnosis code). ICD-10-CM codes are mapped to 22 chapters; ICD-9-CM codes are mapped to corresponding chapter equivalents. Prediction is made at discharge, with target admission diagnoses excluded. Evaluated using macro-averaged AUROC and AUPRC.

\item \textbf{ICD Category Multi-label Classification.} Multi-label classification predicting all 3-character ICD categories present in the admission. We retain categories appearing in at least 100 admissions. Labels are represented as multi-hot vectors over the category vocabulary. Prediction is made at discharge with target diagnoses excluded. Evaluated using macro-averaged AUROC and AUPRC.

\end{itemize}

\paragraph{Data Leakage Prevention.} For mortality and prolonged LoS tasks, we apply temporal truncation to limit input data to the first 24h and 48h after admission, respectively. For all classification tasks except readmission, we exclude diagnosis codes (\texttt{DX:*}) from the target admission to prevent trivial prediction.

\paragraph{Evaluation Metrics.} Binary classification tasks use AUROC and AUPRC as primary metrics. Multi-class and multi-label tasks use macro-averaged AUROC and AUPRC, computed via one-vs-rest per class then unweighted averaging.

\begin{table}[h]
\centering
\caption{Summary of downstream task configurations.}
\label{tab:downstream_tasks}
\small
\begin{tabular}{lcccc}
\toprule
\textbf{Task} & \textbf{Type} & \textbf{Prediction Time} & \textbf{Exclude DX} & \textbf{Primary Metric} \\
\midrule
Mortality         & Binary         & Admission + 24h  & Yes & AUROC / AUPRC \\
30-Day Readmission & Binary        & Discharge        & No  & AUROC / AUPRC \\
Prolonged LoS     & Binary         & Admission + 48h  & Yes & AUROC / AUPRC \\
ICD Chapter       & Multi-class    & Discharge        & Yes & Macro AUROC / AUPRC \\
ICD Category      & Multi-label    & Discharge        & Yes & Macro AUROC / AUPRC \\

\bottomrule
\end{tabular}
\end{table}

\subsubsection{Ablation studies explanation}
\label{appx:mimic-abl}
\textbf{Both encoders are essential, and depth cannot substitute for either.}
Removing CSE degrades all three tasks substantially; doubling SWE depth to match NEST's parameter count recovers only part of the gap. Replacing SWE's parametric attention with parameter-free mean pooling incurs an even larger drop, exceeding the parameter-matched SWE-only variant on every metric. Depth alone cannot substitute on either side.

\textbf{CSE's contribution comes from temporal structure, not from cross-set attention alone.}
Shuffling encounter order at fine-tuning time degrades performance to the level of removing CSE entirely, indicating that CSE's downstream contribution is the temporal structure it encodes rather than cross-set attention itself.

\textbf{The \texttt{[CLS]} pathway carries set-level information that mean pooling cannot recover.}
Replacing \texttt{[CLS]} with mean pooling at fine-tuning time degrades all tasks. Mean pooling can underperform even Freeze CSE despite leaving CSE trainable, indicating that the set-level information MSM places in \texttt{[CLS]} during pretraining cannot be recovered from token averages by refitting CSE.

\textbf{SWE representations require more task-specific adaptation than CSE representations.} Freezing SWE hurts more than freezing CSE on every task: within-encounter representations carry more of the task-specific signal that fine-tuning must adapt, while pretrained cross-encounter temporal patterns transfer with less adjustment.

Together, these ablations show that NEST's gains arise from its architectural commitments rather than capacity or pretraining artifacts.

\subsection{Proprietary EHR}
\label{appx:proprietary-ehr}

\begin{figure}
    \centering
    \includegraphics[width=0.75\linewidth]{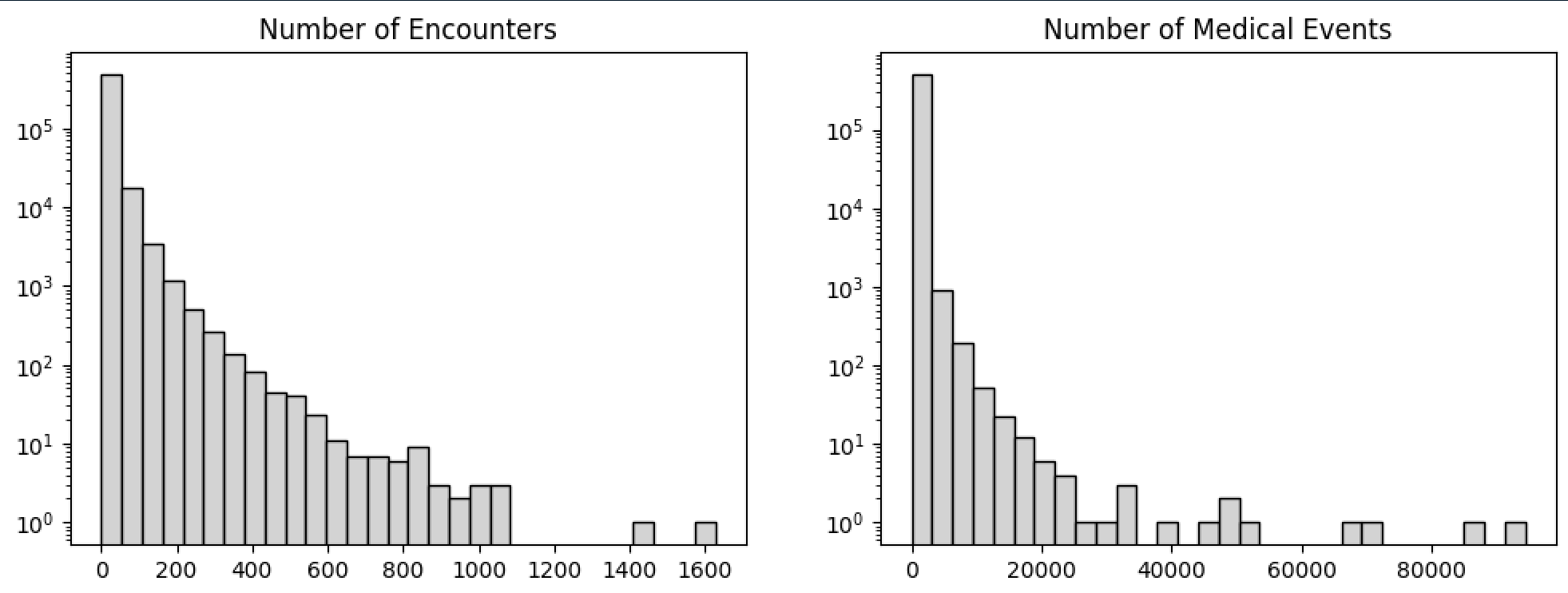}
    \caption{Distributions for \# multisets and \# tokens per patient.}
    \label{fig:ehr-distr}
\end{figure}

The train-test split in the downstream task cohorts follows the deterministic birth-date–based split of the pretrain data. Patients born on the 17th–24th of each month are assigned to the test set and the remainder to the training set. 

\subsubsection{ETL pipeline}

We implemented an ETL data processing pipeline to extract multi-source, heterogeneous EHR tables from a clinical Oracle database conforming to the PCORnet Common Data Model (CDM). The ETL process is organized around patients' longitudinal clinical histories and retrieves structured clinical events from patient demographics, diagnoses, procedures, laboratory measurements, medication administrations and prescriptions, vital signs, and immunizations. All records are joined, filtered, and indexed at both the patient and encounter levels. Clinical records from different domains are subsequently normalized into a shared base schema, consisting of standardized patient and encounter identifiers, a unified event timestamp, and event-type--specific code fields along with the required numerical attributes, enabling consistent event-level representation across domains for downstream sequence modeling.

\subsubsection{Downstream Cohorts}

\paragraph{Recurrent Acute Otitis Media (rAOM)}
rAOM, or recurrent ear infection, is a common medical condition in early childhood. While most children will experience an ear infection during their first years of life, approximately 20-30\% of the pediatric population will experience recurrent ear infections defined as 3 episodes within a 6 month period or 4 episodes within a year period. Children who experience rAOM often require surgical intervention in the form of ear tube placement. We sought to develop a predictive model for the probability a first AOM develops into rAOM. 

We identified a cohort of 3,852 children with a first AOM. Of these, 22.0\% transitioned into rAOM. We used the same embeddings framework as above to process prenatal, birth and developmental clinical data. Since our eligible cohort had reached age 4 --- an age at which rAOM is unlikely to developed --- we modeled the binary diagnosis indicator as the outcome. Our training and testing sizes were 2,822 and 1,030, respectively.

\paragraph{Autism Spectrum Disorder (ASD)}
ASD is a neuro-developmental condition that typically presents within the first years of life. The mean age of the diagnosis is around 5 years. Children are typically screened during their 18-month well-child visit and referred for confirmatory diagnosis. Previous works has shown that early medical conditions (e.g., gastrointestinal disease) can serve as early indicators of a future ASD diagnosis. Our team has been focusing on developing automated early screening tools for ASD.

We identified a cohort of 43,945 children who had a well-child visit at our institution between 12 and 24 months of age. We used the well-child visits closest to 18 months as the index encounter.  We required two distinct medical encounters with an ICD-10 for ASD to label an ASD case. Across our cohort's follow-up, 2.05\% were diagnosed with ASD. Our training and testing sizes were 35,000 and 8,945, respectively.

\paragraph{30-day Hospital Readmission}
Hospital readmission is a commonly used measure of short-term healthcare utilization and quality of care in pediatric populations. We focus on 30-day hospital readmission following emergency inpatient or inpatient hospitalization encounters among children.

We identified a cohort of 26,617 children who had at least one emergency inpatient (EI) or inpatient hospitalization (IP) encounter at our institution between January 1, 2016 and age 18. After excluding birth-related encounters, we used all qualifying EI or IP encounters with a discharge disposition of Home/Self Care (HO) or Home Health (HH) as index encounters, allowing multiple index encounters per patient. We defined the outcome as 30-day readmission, requiring the occurrence of a subsequent Emergency Department (ED), EI, or IP encounter within 30 days following discharge from the index encounter. The final cohort consisted of 45,755 index encounters, of which 19.1\% were followed by a 30-day readmission. Our training and testing sizes were 33,683 and 12,072, respectively.

\section{Implementation Details}
\label{app:implementation_details}

This section reports implementation settings used in pretraining, downstream fine-tuning, and ablation experiments. To ensure fair comparison, we standardize key hyperparameters across models where possible: all MLM-based pretraining uses mask ratio \texttt{0.20}, and all downstream fine-tuning uses learning rate \texttt{1e-5}. Per-model differences in scheduler, warmup, weight decay, and pretraining objective reflect each baseline's original protocol where deviation would alter the model's identity. Vocabulary size is read at runtime from the tokenizer; in our checkpoints this is \texttt{20,377} (CEHR-BERT adds 18 extra special tokens for \texttt{[VS]}, \texttt{[VE]}, \texttt{W0--W3}, \texttt{M1--M11}, \texttt{LT}, yielding \texttt{20,395}).

\begin{table*}[t]
\centering
\small
\caption{Backbone architecture and pretraining setup, including objective type, masking policy, warmup, and loss weighting used in reported runs.}
\label{tab:impl_backbone}
\resizebox{\linewidth}{!}{
\begin{tabular}{p{2.0cm}p{2.0cm}p{4.5cm}p{4.5cm}p{1.8cm}}
\toprule
Model & Objective & Key pretraining settings & Core architecture & Scale \\
\midrule
NEST (default) & MLM + MSM & token mask \texttt{0.20}, encounter mask \texttt{0.40}, loss weights MLM\,$=$\,MSM\,$=$\,\texttt{1.0}, no warmup & \texttt{d\_model=768}, \texttt{n\_blocks=6}, \texttt{n\_heads=12}, \texttt{d\_ff=2048}, SWE+RoPE, T2V (\texttt{scale=1.0}), \texttt{max\_seg=8}, \texttt{max\_seq\_len=512} & $\sim$120M \\
NEST (token mode) & MLM only & token mask \texttt{0.20}, MSM weight \texttt{0}, no warmup & Same backbone as default NEST & $\sim$120M \\
CORE-BEHRT & MLM & mask \texttt{0.20}, no warmup & \texttt{768/12/12/2048}, \texttt{max\_seq\_len=2048} & $\sim$116M \\
HEART$^\dagger$ & MLM & mask \texttt{0.20}, no warmup & \texttt{768/10/12/2048}, edge-aware attention, \texttt{max\_seq\_len=768} & $\sim$117M \\
MOTOR & TTE & \texttt{n\_tasks=512}, \texttt{n\_time\_bins=7}, \texttt{final\_layer\_size=32}, no warmup & Encoder \texttt{768/12/12/2048} & $\sim$116M (enc) \\
CEHR-BERT & MLM & mask \texttt{0.20}, warmup ratio \texttt{0.06} & \texttt{hidden=768}, \texttt{layers=12}, \texttt{heads=12}, \texttt{intermediate=3072}, ATT tokens, \texttt{n\_time\_embd=16} & $\sim$117M \\
GT-BEHRT & NAM $\rightarrow$ MNP+VTP & NAM mask \texttt{0.20} (10 epochs), VTP mask \texttt{0.50}, warmup ratio \texttt{0.1} & \texttt{hidden=780}, \texttt{graph\_layers=3}, \texttt{bert\_layers=12}, \texttt{heads=12} & $\sim$118M \\
Hi-BEHRT (BYOL) & BYOL & segment mask \texttt{0.20}, momentum \texttt{0.996}, warmup ratio \texttt{0.1} & \texttt{d\_model=768/9+10/12/2048}, extractor/aggregator hierarchy, \texttt{window=50}, \texttt{stride=30}, \texttt{t2v\_dim=64} & $\sim$117M (BYOL: trainable parameters; target network is an EMA copy) \\

\bottomrule
\end{tabular}
}
\\[0.4em]
{\footnotesize $^\dagger$HEART configurations are reported for completeness; the model was excluded from main results due to OOM failures from its $O(\sum_i n_i^2)$ pairwise edge embeddings (see Section~\ref{sec:mimic-iv-ablation}).}
\end{table*}

\begin{table*}[t]
\centering
\small
\caption{Optimization protocol for pretraining. LR schedulers reflect each model's original protocol; models without a scheduler use a constant LR after any warmup.}
\label{tab:impl_optim_clean}
\resizebox{\linewidth}{!}{
\begin{tabular}{p{3.0cm}p{2.0cm}p{3.5cm}p{1.8cm}p{3.0cm}}
\toprule
Model family & LR & Scheduler & Epochs & Batch size \\
\midrule
NEST & \texttt{5e-5} & none & \texttt{30} & \texttt{36} \\
CORE/HEART & \texttt{5e-5} & none & \texttt{30} & CORE: \texttt{36}; HEART: \texttt{36} \\
MOTOR & \texttt{5e-5} & none & \texttt{30} & \texttt{36} \\
CEHR-BERT & \texttt{3e-5} & linear warmup (\texttt{0.06}) + cosine decay & \texttt{30} & \texttt{36} \\
GT-BEHRT & \texttt{5e-5} & linear warmup (\texttt{0.1}) + linear decay & NAM \texttt{10} + main \texttt{20} & \texttt{36} \\
Hi-BEHRT (BYOL) & \texttt{5e-5} & linear warmup + cosine decay & \texttt{30} & \texttt{36} \\
\bottomrule
\end{tabular}
}
\end{table*}

\begin{table*}[t]
\centering
\small
\caption{Downstream fine-tuning protocol used across all models and tasks.}
\label{tab:impl_finetune}
\resizebox{\linewidth}{!}{
\begin{tabular}{p{4.0cm}p{10.0cm}}
\toprule
Setting & Value \\
\midrule
Optimizer & AdamW (weight decay: \texttt{0.001}) \\
Learning rate & \texttt{1e-5} \\
Scheduler & linear warmup + cosine decay \\
Epochs & \texttt{30} \\
Effective batch size & model-specific, following each training entrypoint default (reported in Table~\ref{tab:impl_optim_clean}) \\
Classification head & two-layer MLP on top of the last-segment \texttt{[CLS]} representation (or mean-pooled non-\texttt{[CLS]} tokens for the \texttt{[CLS]}\,$\to$\,mean pool ablation), with model-specific activation following each model's original implementation: GELU for NEST, CORE-BEHRT, MOTOR, CEHR-BERT, and GT-BEHRT; ReLU for HEART; Tanh for Hi-BEHRT \\
Loss function & cross-entropy for binary and multi-class tasks; binary cross-entropy with logits for multi-label tasks (Hi-BEHRT uses cross-entropy for multi-label, matching its original implementation) \\
Early stopping metric & validation AUPRC for binary tasks (all models); for multi-class and multi-label tasks, NEST uses validation AUPRC while baselines use validation AUROC, following each model's training-script default \\
Backbone treatment & full fine-tuning by default; ablation variants (\texttt{Freeze SWE}, \texttt{Freeze CSE}, \texttt{Linear probe}) freeze specified components \\
Data split and seed control & Following the NEST protocol as the unified reference: patient-level split with fixed seed \texttt{42}; pretraining uses \texttt{80/10/10} and downstream fine-tuning uses \texttt{70/15/15} (train/val/test). \\
Early stopping patience & \texttt{10} for pretraining and \texttt{3} for downstream fine-tuning. \\
Mixed precision detail & CUDA AMP with FP16 (autocast + GradScaler); bf16 is not used. \\
\bottomrule
\end{tabular}
}
\\[0.4em]
{\footnotesize Any metric, loss, or weight-decay differences above are inherited from the original model-specific training scripts and are reported explicitly for transparency.}
\end{table*}

\begin{table*}[t]
\centering
\small
\caption{Task and seed protocol. All reported ablation means use 5 seeds.}
\label{tab:impl_task_seed}
\resizebox{\linewidth}{!}{
\begin{tabular}{p{3.3cm}p{7.0cm}p{4.2cm}}
\toprule
Category & Scope & Seeds \\
\midrule
Downstream tasks & \texttt{mortality}, \texttt{readmission\_30d}, \texttt{prolonged\_los}, \texttt{icd\_chapter}, \texttt{icd\_category\_multilabel} & Per-task runs use fixed seed control via entrypoints \\
Ablation reporting protocol & Group A + Group B results are averaged over fixed random seeds & \texttt{42, 123, 456, 2025, 3407} \\
\bottomrule
\end{tabular}
}
\end{table*}

\paragraph{Hardware and mixed precision.}
All runs use one NVIDIA H200 GPU, 10 CPU cores, and 36GB host memory per job. Mixed precision is enabled via \texttt{--use\_amp} (CUDA AMP autocast with FP16 + GradScaler).

\section{Ablation Setup}
\label{appx:ablation-setup}

\subsection{Group B: Architectural Ablations}
Group B weakens internal NEST components and re-pretrains from scratch under matched data and optimization budget. The variants are:
\begin{itemize}[topsep=2pt, itemsep=2pt]
    \item \textbf{$-$\,CSE (6L/12L SWE-only):} CSE blocks removed entirely. The 12L variant is parameter-matched to NEST.
    \item \textbf{SWE\,$\to$\,MeanPool (6L/12L):} SWE's attention and FFN are replaced with parameter-free mean pooling over non-\texttt{[CLS]} tokens, with CSE preserved. A broadcast residual propagates cross-encounter context back to all token positions.\footnote{A direct ``remove SWE'' variant is infeasible since CSE operates on \texttt{[CLS]} tokens produced by SWE; MeanPool is the closest empirical proxy.} The 12L variant is parameter-matched to NEST.
\end{itemize}

\subsection{Group A: Behavioral Ablations}
Group A keeps the NEST (MLM+MSM) checkpoint fixed and intervenes only at fine-tuning, isolating individual inductive biases:
\begin{itemize}[topsep=2pt, itemsep=2pt]
    \item \textbf{Linear probe:} both encoders frozen; only the classification head is trained. Lies below both Freeze variants on every metric, confirming that fine-tuning some backbone component is necessary for full task adaptation.
    \item \textbf{Temporal shuffle:} encounter order randomly permuted before fine-tuning.
    \item \textbf{\texttt{[CLS]} $\to$ mean pool:} downstream classifier reads mean-pooled non-\texttt{[CLS]} tokens instead of \texttt{[CLS]}.
    \item \textbf{Freeze SWE / Freeze CSE:} one encoder frozen, the other trainable.
\end{itemize}

\section{Attention Pattern Visualization}
\label{sec:attn-viz}
\subsection{Hierarchical Pattern}
\begin{figure}[H]
    \centering
    \includegraphics[width=0.5\linewidth]{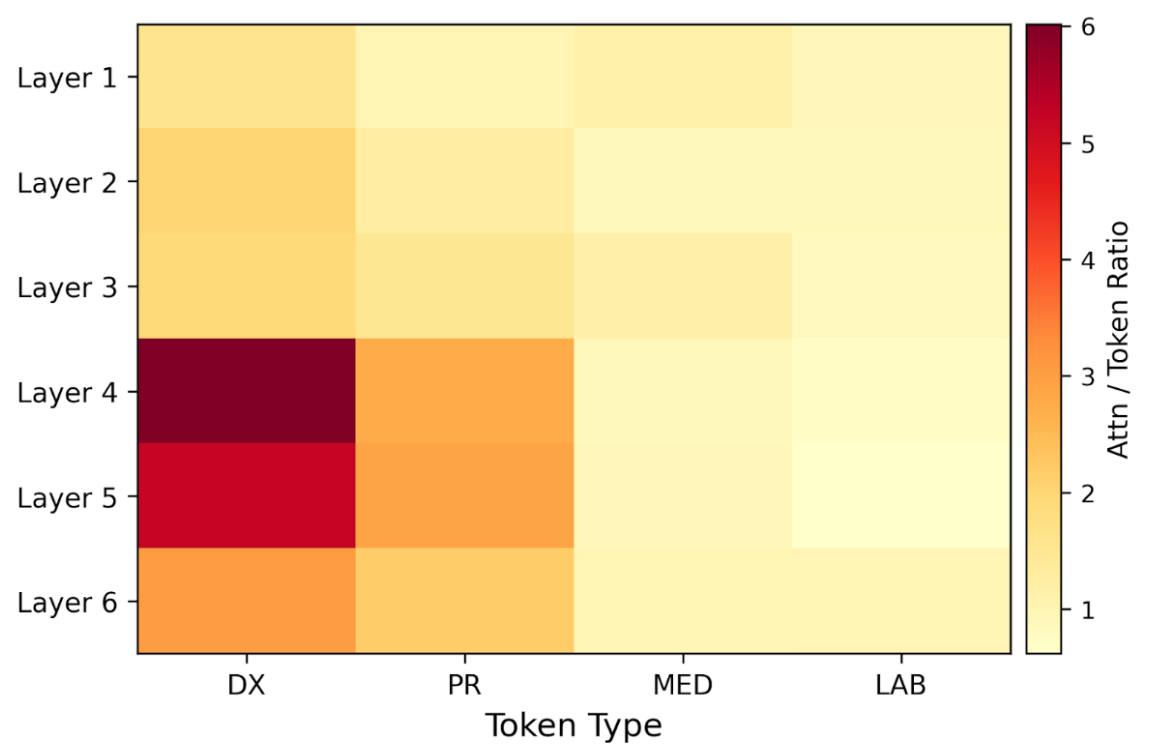}
    \caption{NEST captures hierarchical structure via within-set attention in the SWE module. The heatmap displays the deviation of \texttt{[CLS]} attention weights from a uniform baseline across context tokens using the MIMIC-IV data. \texttt{DX}: diagnosis; \texttt{PR}: procedure; \texttt{MED}: medication; \texttt{LAB}: lab measures.}
    \label{fig:hi-attn}
\end{figure}

\subsection{Temporal Pattern}
\begin{figure}[H]
    \centering
    \includegraphics[width=0.8\linewidth]{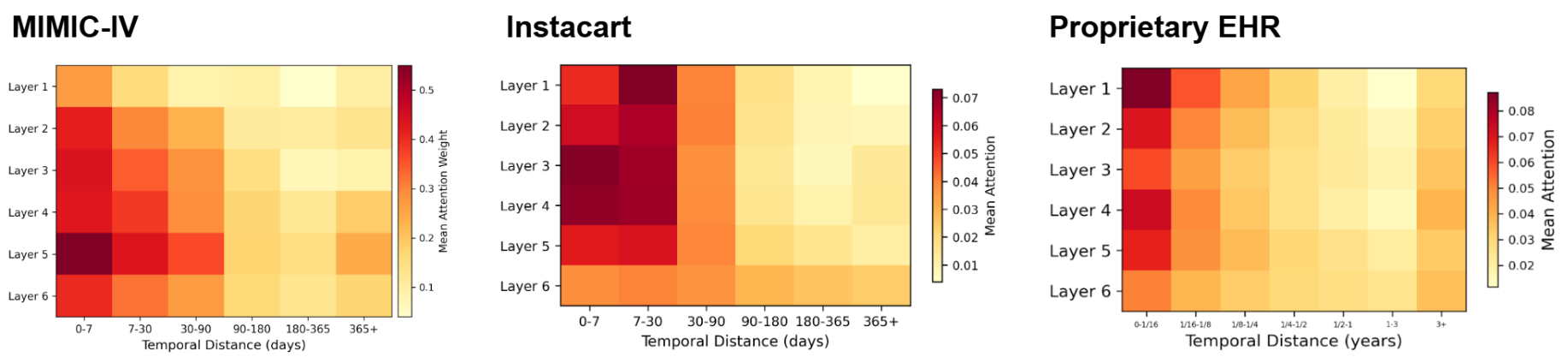}
    \caption{CSE learns temporal recency pattern among the \texttt{[CLS]} tokens.}
    \label{fig:t-attn}
\end{figure}



\end{document}